\pdfoutput=1
\documentclass{article}
\usepackage{graphicx} 
\usepackage{geometry}
\usepackage{diagbox}
\usepackage{authblk}
\usepackage{hyperref}
\usepackage{natbib}
\usepackage{algorithm}
\usepackage{algpseudocode}
\usepackage{notation}
\usepackage{float}
  
\title{Contextual Online Decision Making with Infinite-Dimensional Functional Regression}

\author[1,4]{Haichen Hu\thanks{{Email: \texttt{huhc@mit.edu}}}} 
\author[2]{Rui Ai\thanks{{Email:\texttt{ruiai@mit.edu}}}} 
\author[3]{Stephen Bates\thanks{{Email:\texttt{stephenbates@mit.edu}}}}
\author[2,4]{David Simchi-Levi\thanks{{Email:\texttt{dslevi@mit.edu}}}}
\affil[1]{Center for Computational Science and Engineering, MIT} 
\affil[2]{Institute for Data, Systems, and Society, MIT} 
\affil[3]{Department of EECS, MIT}
\affil[4]{Department of Civil and Environmental Engineering, MIT}
\date{}
\begin{document}
\maketitle
\begin{abstract}
    Contextual sequential decision-making problems play a crucial role in machine learning, encompassing a wide range of downstream applications such as bandits, sequential hypothesis testing and online risk control. These applications often require different statistical measures, including expectation, variance and quantiles.
    In this paper, we provide a universal admissible algorithm framework for dealing with all kinds of contextual online decision-making problems that directly learns the whole underlying unknown distribution instead of focusing on individual statistics. This is much more difficult because the dimension of the regression is uncountably infinite, and any existing linear contextual bandits algorithm will result in infinite regret. To overcome this issue, we propose an efficient infinite-dimensional functional regression oracle for contextual cumulative distribution functions (CDFs), where each data point is modeled as a combination of context-dependent CDF basis functions.
Our analysis reveals that the decay rate of the eigenvalue sequence of the design integral operator governs the regression error rate and, consequently, the utility regret rate. Specifically, when the eigenvalue sequence exhibits a polynomial decay of order $\frac{1}{\gamma}\ge 1$, the utility regret is bounded by $\tilde{\mathcal{O}}\Big(T^{\frac{3\gamma+2}{2(\gamma+2)}}\Big)$. By setting $\gamma=0$, this recovers the existing optimal regret rate for contextual bandits with finite-dimensional regression and is optimal under a stronger exponential decay assumption. Additionally, we provide a numerical method to compute the eigenvalue sequence of the integral operator, enabling the practical implementation of our framework.
\end{abstract}

\section{Introduction}\label{sec:intro}
Contextual sequential online experimentation has been overwhelmingly important for online platforms, healthcare companies, and other businesses~\citep{tewari2017ads,saha2020improved,beygelzimer2011contextual,avadhanula2022bandits}.
The tasks that a decision maker (DM) might face range widely, and different tasks lead to various adaptive algorithm designs. For example, for maximizing the total reward, there are all kinds of contextual bandit algorithms~\citep{chu2011contextual}. For minimizing the total hypothesis testing error, we have online hypothesis testing algorithms~\citep{wei2007passive}. To rectify a pre-trained machine learning model, many online calibration algorithms have been developed~\citep{fasiolo2021fast}.

Some common structures appear to be obscured behind these examples. At each round, the DM first receives some context, such as past consumption records or symptoms in a healthcare setting. Then, the DM chooses an action to apply according to this context and his objective. For example, in online hypothesis testing, the actions could be ``reject'' and ``accept'', while in contextual bandits, the action is to choose and pull an arm. Commonly, there is an underlying distribution $P_{x,a}$ associated with every context $x$ and action $a$. The DM makes decisions based on the estimation of its various statistics. In the context of bandits, the primary statistic of interest is the expectation, whereas in risk control, attention is directed toward variance~\citep{li2022novel}. Please refer to \citet{ayala2024risk,bouneffouf2016contextual,sun2016risk,li2019sequential,han2021truthful,kong2024peer} for more examples.

The essential difficulty behind all these online decision-making examples is the uncertainty of the outcomes under different actions. In regret minimization, the decision maker is not sure about the mean reward, while in risk-aware bandits, he is uncertain about both the mean and variance. Moreover, in online quantile calibration, the key concern turns to be the unknown quantile function. See ~\citet{gupta2021online,bastani2022practical} for example.

Note that no matter what kind of statistic we aim to know, as long as we learn the distribution $P_{x,a}$, the decision maker knows exactly how to make the decision. Therefore, it raises a question,

\begin{center}
\textit{
Is it possible to consider a more general (robust and adaptive) online decision-making setting based on infinite-dimensional functional regression?
}
\end{center}
Specifically, we are looking for a general solution with infinite-dimensional functional regression of cumulative distribution functions (CDFs).

Our results are presented in the following structure. In \Cref{sec:model}, we introduce our general decision-making model based on functional regression and some math tools. In \Cref{sec:functional_regression}, we put forward an efficient functional regression algorithm and establish its oracle inequality. Finally, in \Cref{sec:alg}, we provide our sequential decision-making algorithm and its theoretical guarantee.
\subsection{Related Works}
Our work is intimately related to the lines of work on contextual bandits, operation learning and functional regression.

{\bf Contextual Bandits}
Contextual bandits have been widely studied in both academia and industry, including applications such as recommendation systems~\citep{tewari2017ads} and healthcare~\citep{zhou2023spoiled}. Please refer to~\citet{zhou2015survey,bouneffouf2020survey} for comprehensive surveys.
However, most theoretical results require assumptions about the dimensionality, such as the context having a finite cardinality~\citep{agarwal2014taming} or the existence of a low-dimensional representation, as in linear contextual bandits~\citep{chu2011contextual}.
Our paper is the first to address the problem of contextual decision-making based on infinite-dimensional functional regression. We establish universal regret bounds for various tasks under eigenvalue decay assumptions and further derive general sublinear regrets without them.

{\bf Operator Learning.} 
Operator learning~\citep{mollenhauer2022learning,kovachki2024operator,adcock2024learning} studies the case where the input and output are elements in Banach spaces such as some function spaces. However, in our paper, although the objective is the distribution function, we are unable to directly observe the function itself. Instead, we can only observe one single data point sampled from this distribution. In other words, we do not have full feedback but only bandit feedback, which makes our problem much more difficult than general operator learning~\citep{foster2018practical,foster2020beyond}.

{\bf Functional Regression.}
Functional data analysis focuses on analyzing data where each observation is a function defined over a continuous domain, typically sampled discretely from a population. \citet{morris2015functional,ramsay2002applied} provide comprehensive introductions to functional regression. 
One subdomain, Kernel learning, for example, learning in reproducing kernel Hilbert space (RKHS) has been a hot topic in recent years~\citep{yeh2023sample,hou2023sparse}. 
Nonetheless, the problem we focus on is quite different from theirs, as datasets in kernel learning contain input and output points lying in some kernelized spaces, while we do not have access to data of required integral kernel $\cL^{x,a}$ lying in an infinite-dimensional space $L^2(\Omega,\nu)$ directly. ~\citet{zhang2022functional,azizzadeneshelisparse} studies functional regression for contextual CDFs in finite dimensions. For infinite dimensions, under some special conditions, ~\citet{zhang2022functional} also proposes a method based on some regularization term which scales with respect to the number of data points.

\subsection{Our Contributions}
We summarize our contribution into the following three main cornerstones.

{\bf A unified framework: Arbitrary contextual online decision-making problem.}
We propose a unified framework for contextual online decision-making capable of addressing a wide range of tasks. Within our framework, various problems, such as online hypothesis testing, bandits, and risk control, can be addressed by substituting the integral operator module, while ensuring sublinear regret universally  across all tasks. 
Notably, this framework facilitates the translation of any given functional and CDF basis family into an efficient algorithm for minimizing utility regret.
Furthermore, it streamlines the design of algorithms for real-world applications. By replacing the operator as needed, our framework adaptively achieves the desired objectives. 

{\bf An Efficient regression oracle: Infinite-dimensional functional regression.} 
Our algorithm is not constrained by dimensionality, surpassing existing approaches that often require assumptions on embedding space~\citep{zhu2022contextual} and function approximation~\citep{jin2020provably}. This advancement provides a new perspective for analyzing and interpreting the performance of large-scale models. Additionally, our algorithm only requires 
$\cO(\log T)$ calls to functional regression, resulting in low computational complexity. This efficiency makes it more suitable for large-scale deep learning models and approximation algorithms.

{\bf A new insight: Linking regret and eigendecay rate.}
We provide a rigorous characterization of the relationship between regret and the eigendecay rate of the operator, depicted by a single parameter $\gamma$. Specifically, we derive a regret bound of $\tilde{\cO}\rbr{T^{\frac{3\gamma+2}{2(\gamma+2)}}}$, offering new insights into the complexity of learning as a function of the operator's spectral decay. For finite-dimensional problems, where $\gamma=0$, we can recover the optimal regret bound $\cO(\sqrt{T})$ in the existing literature ~\citep{chu2011contextual} up to logarithmic factors. On the other hand, in scenarios where prior knowledge of $\gamma$ is unavailable, the algorithm achieves a regret bound of $\Tilde{O}(T^{\frac{5}{6}})$, highlighting its robustness and potential for deployment in real-world environments with incomplete information.

{\bf Notation:} For any measure space $(B,\cB,n)$, we use $L^2(B,n)$ to denote the square-integrable function space which is also a Hilbert space with inner product $\langle\cdot,\cdot\rangle=\int_{B} f(x) g(x)dn(x)$ and norm $||f||_{L^2(B,n)}=\sqrt{\inner{f}{f}}$. $\cO(\cdot)$ and $o(\cdot)$ stand for Bachmann–Landau asymptotic notations up to constants. Meanwhile, $\Tilde{\cO}$ stands for the asymptotic notations up to logarithmic terms. For any $n\in\NN$, $[n]$ denotes the set $\cbr{1,\cdots,n}$, and $\II_{y}(t)$ denotes the indicator function $\II\cbr{y\le t}$. Unless otherwise stated, when we write the eigenvalue sequence $\cbr{\lambda_i}$, it is arranged in a decreasing order. 
\section{Model setup and Operator Eigendecay}\label{sec:model}

\subsection{Model Setup}\label{subsec:model}
There are two main bodies of our model, \emph{Functional Regression Model} and \emph{Decision-Making Model.} We first focus on the functional regression model.

\paragraph{Functional-Regression Model:} Assume a feature space $\cX$ and a finite action space $\cA=\cbr{a_1,\cdots,a_K}$. For any context $x\in \cX$ and action $a$, there is an associated random variable $Y_{a,x}$ which takes values in some bounded Borel set $S\subset R$ with $m(S)=1$, where $m$ is some measure on the real line. We assume $Y_{a,x}$ has a cumulative distribution function $F^*(x,a,s)$ satisfying the following \Cref{ass:regression_model} and \Cref{assump:Lip_cts_phi}.

\begin{assumption}\label{ass:regression_model}
We have access to a function family of CDF basis $\Phi=\cbr{\phi(x,a,w,s)}_{w\in\Omega}$ indexed by a compact set $\Omega\subset \RR^d$ with $d$-dimensional Lebesgue measure $\nu(\Omega)=1$. That is, given any $x\in\cX,a\in\cA,w\in\Omega$, $0\le\phi(x,a,w,s)\le 1$ is a CDF of some $S$ valued random variable. For any $x,a$, there is an unknown non-negative coefficient function $\theta^*\in L^2(\Omega,\nu)$ such that
\[F^*(x,a,s)=\int_{\Omega}\theta^*(w)\phi(x,a,w,s)d\nu(w)=\inner{\theta^*(\cdot)}{\phi(x,a,\cdot,s)}.\]
We also set $\int_{\Omega}\theta^*(w)d\nu(w)=1$ to ensure that $F^*(x,a,s)$ is a cumulative distribution function. For boundedness, we assume that $||\theta^*||_{L^2(\Omega,\nu)}\le M$ for some constant $M$.
\end{assumption}
Our goal in functional regression model is to derive an accurate coefficient function estimation $\hat{\theta}_{\cD}$.
\begin{assumption}\label{assump:Lip_cts_phi}
    $\phi(x,a,\cdot,s)$ is $L_0$-Lipschitz continuous,
    $|\phi(x,a,w,s)-\phi(x,a,r,s)|\le L_0||w-r||_{\infty}.$
\end{assumption}
Note that the dimension of our functional regression problem is uncountably infinite because the family $\Phi$ may contain uncountably many functions, making the problem significantly harder than finite-dimensional linear regression and hindering the applicability of methods used in the latter.

We now turn to our Decision-Making Model.
\paragraph{Decision-making Model:}
To incorporate all the sequential decision-making examples in our framework, we assume that the objective of the decision-maker is relevant to the cumulative distribution function itself, rather than only some concrete statistics such as mean, variance, and quantile.

To be specific, we assume that there is a known functional $\cT$ defined on function space $L^2(S,m)$. At each round $t$, the context $x_t$ is drawn i.i.d. from some unknown distribution $Q_{\cX}$. Given any context $x$, if the DM applies action $a$, then one data point will be sampled and observed according to distribution function $F^*(x,a,s)$, and the 
utility of action $a$ under $x$ is $\cT(F^*(x,a,s))$. The goal of the decision maker is to maximize the total utility in $T$ rounds of interaction, i.e., $\max_{\cbr{a_t}_{t=1}^{T}}\sum_{t=1}^{T}\cT(F^*(x_t,a_t,s))$.

Analogous to the bandit setting, the policy $\pi$ is a mapping from the feature space $\cX$ to the action space probability simplex $\Delta(\cA)$. Given any $x$, we denote $a^*(x)$ as the action that maximizes the utility, i.e., $a^*(x)=\argmax_{a}\cT(F^*(x,a,s))$. For any policy sequence $\cbr{\pi_t}_{t=1}^{T}$, the performance measure is utility regret, which is the difference between the optimal policy $\pi^*(x)=\delta_{a^*(x)}(\cdot)$\footnote{$\delta_{a}$ is the dirac function of $a$, i.e., $\delta_{a}(x)=1$ only if x=a, othwise, $\delta_a(x)=0$} and $\cbr{\pi_s}_{s=1}^{T}$.
\[\text{Reg}(\cbr{\pi_t}_{t=1}^{T}):=\sum_{t=1}^{T}\cT(F^*_{x_t,a^*(x_t)})-\cT(F^*_{x,\pi_t(x)}),\]
For a graphical illustration, please see
\cref{fig:illustration}.
\begin{figure}[ht]
\vskip 0.2in
    \begin{center}
\centerline{\includegraphics[width=0.8\linewidth]{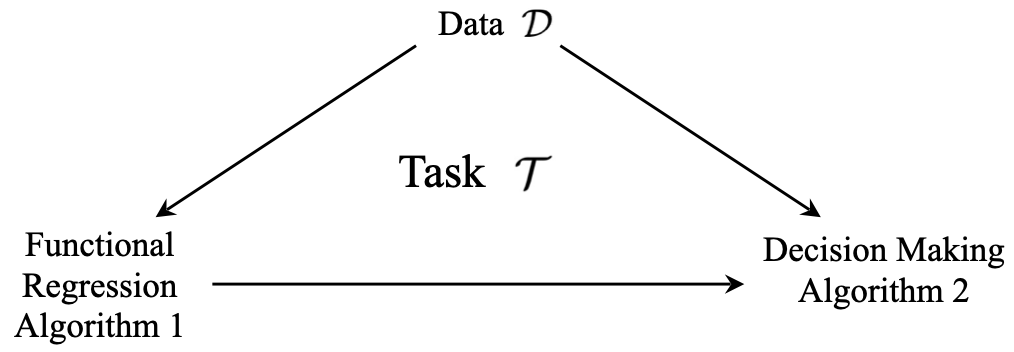}}
\caption{Data $\cD$ drives the functional regression and decision making\label{fig:illustration}.}
\end{center}
\vskip -0.2in
\end{figure}
\begin{example}[Contextual Bandits]
    Define the functional as $\cT: F\rightarrow\int_{S}xdF(x)$, which is a known functional in terms of $F$. Given any context $x$, $F^*(x,a,s)$ is the conditional reward distribution of arm $a$ given context $x$. Consequently, we recover the contextual bandit problem.
\end{example}
 
\begin{example}[Sequential Hypothesis Testing ~\citep{naghshvar2013active}] There are $M$ hypotheses of interest $\cbr{H_1,\cdots,H_M}$ among which only one holds true. The action $a_i$ is to choose which hypothesis $H_i$ to believe. The distribution associated with observation $x$ and action $a_i$ is a posterior multinomial distribution characterized by $F^*_{x,a_i}=\cbr{\PP(H_1|x,a_i),\PP(H_2|x,a_i),\cdots,\PP(H_M|x,a_i)}$. $\PP(H_j|x,a_i)$ is the posterior probability that $H_j$ is true given observation $x$ and that we choose $a_i$. Note that in fact, no matter which $a_i$ we apply, it does not impact the posterior distribution given observation, i.e., $\forall i,\ \PP(H_j|x,a_i)=\PP(H_j|x)$. We define loss vector $L_{i}=\cbr{L_{i,1}\cdots,L_{i,M}}$. $L_{i,j}$ is the penalty of choosing $H_i$ when $H_j$ is the underlying truth. $\cT:F^*_{x,a_i}\rightarrow \sum_{j=1}^{M}\PP(H_j|x,a_i)L_{i,j}$ is the true expected penalty of choosing $H_i$, and the DM wants to minimize it, which is equivalent to maximizing its additive inverse.
\end{example}

\begin{assumption}\label{ass:lip_cts_functional}
We assume that the known functional $\cT$ is $L$-Lipschitz continuous with respect to the norm $||\cdot||_{L^2(S,m)}$.    
\end{assumption}
It is intuitive that a good estimation of the underlying distribution may result in a good utility estimation. However, estimating the underlying CDF with an uncountable basis and arbitrary feature space is highly non-trivial. In \Cref{subsec:operator_eigendecay}, we would like to introduce some tools and assumptions from the functional analysis and operator theory community before we dive into our regression oracle. For background knowledge of operator theory and functional analysis, please refer to some textbooks ~\citet{conway2000course,kubrusly2011elements,simon2015operator}.

\subsection{Operator Eigendecay}\label{subsec:operator_eigendecay}
    
\begin{definition}\label{def:integral_operator_kernel}
    For any $(w,r)\in\Omega\times\Omega$, define the following linear integral operator, integral kernel, and mapping
    \begin{align*}
        \cL^{x,a}: & \theta(w)\mapsto\int_{\Omega}\int_{S}\theta(r)\phi(x,a,w,s)\phi(x,a,r,s)dm(s)d\nu(r),\\
        \cK^{x,a}: &(w,r)\mapsto\int_{S}\phi(x,a,w,s)\phi(x,a,r,s)dm(s),\\
        \psi_{x,a}: &L^2(\Omega,\nu)\times S\rightarrow R,\ (\theta,s)\mapsto\inner{\theta(\cdot)}{\phi(x,a,\cdot,s)}.
    \end{align*}
\end{definition}
    By Fubini's theorem, we have,
 \[\forall\ w\in\Omega,\ \cL^{x,a}(\theta)(w)=\int_{\Omega}\theta(r)\cK^{x,a}(w,r)d\nu(r)=\int_{S}\phi(x,a,w,s)\psi_{x,a}(\theta,t)dm(s).\]

\begin{theorem}\label{thm:math_property_ourintegral_operator}
For any $x\in\cX$ and $a\in\cA$, $\cL^{x,a}$ is a linear, positive, self-adjoint Hilbert-Schmidt integral operator with $||\cL^{x,a}||\le 1$. Moreover, it is also compact.
\end{theorem}

By the Riesz-Schauder theorem and the spectral decomposition theorem \citep{reed1978iv}, if we denote $\cbr{\lambda_i(\cL)^{x,a}}_{i=1}^{\infty}$ as its eigenvalue sequence in decreasing order and $\cbr{e_i(\cL^{x,a})}$ as the corresponding eigenfunctions, then it holds that
$$\cL^{x,a}(\theta)=\sum_{i=1}^{\infty}\lambda_i(\cL^{x,a})\inner{\theta}{e_i(\cL^{x,a})}e_i(\cL^{x,a}),$$
where $\lambda_1(\cL^{x,a})\ge\lambda_2(\cL^{x,a})\ge\cdots> 0$ and $\lambda_i(\cL^{x,a})\rightarrow 0$.
The eigenvalues and eigenfunctions play a crucial role in our algorithm and analysis. Hence, we further examine them and state additional properties below. All proofs are deferred to \Cref{sec:proofs_model}.

By~\citet{gohberg2012traces,ferreira2013eigenvalue}, we call the Hilbert-Schmidt integral operator $\cL^{x,a}$ traceable if $\sum_{i=1}^{\infty}\lambda_i(\cL^{x,a})<\infty$. For any traceable operator, we could analyze the decay rate of its eigenvalue sequence, and the eigenvalue decay rate of these integral operators has been well studied in functional analysis~\citep{ferreira2013eigenvalue,volkov2024optimal,levine2023decay,carrijo2020approximation}.

Specifically, for our integral operator $\cL^{x,a}$, we observe the following properties.
\begin{property}\label{prop:eigendecay_cL}
    For any $x,a$, the integral operator $\cL^{x,a}$ satisfies
    \begin{itemize}
        \item $\sum_{i=1}^{\infty}\lambda_{i}(\cL^{x,a})=\int_{\Omega}\int_{S}\phi(x,a,w,s)^2dm(s)d\nu(w)\le1$.
        \item For any $i$, $\exists C>0$ such that $\lambda_i(\cL^{x,a})\le \frac{C}{i}$.
        \item $\lambda_i(\cL^{x,a})=o(\frac{1}{i})$.
    \end{itemize}
\end{property}
In the analysis of many machine learning algorithms, in addition to \Cref{prop:eigendecay_cL}, stronger assumptions, such as polynomial or exponential eigendecay, are often proposed~\citep{yeh2023sample,vakili2024kernelized,goel2017eigenvalue,agarwal2018effective}. In our paper, we only assume a weaker polynomial eigendecay condition, namely, the $\gamma$-dominating eigendecay condition. Moreover, our result could be easily extended to exponential eigendecay and obtain an optimal rate. 
\begin{assumption}[$\gamma$-dominating eigendecay]\label{ass:dominating_eigendecay}\footnote{Actually, under some mild conditions, \Cref{ass:dominating_eigendecay} can be proved~\citep{carrijo2020approximation}.}
    Denoting all the bounded linear operators on $L^2(\Omega,\nu)$ as $\cB(L^2(\Omega,\nu))$, we assume that the set $\cbr{\cL^{x,a}:x\in\cX,a\in\cA}\subset\cB(L^2(\Omega,\nu))$ is convex and satisfies the existence of a sequence $\cbr{\tau_i}_{i=1}^{\infty}$ such that \begin{itemize}
    \item for some $0<\gamma\le1$,
$\sum_{k=1}^{\infty}\tau_k^\gamma\le s_0<\infty,$
        \item for $\forall x,a,\ \lambda_k(\cL^{x,a})< \tau_k \le \cO(\frac{1}{k}),\ \forall k.$
    \end{itemize}
\end{assumption}
We name the sequence $\cbr{\tau_i}_{i=1}^{\infty}$ as `$\gamma$-dominating sequence'.
Intuitively, \Cref{ass:dominating_eigendecay} states that the decay rate of the eigenvalue sequence of any operator $\cL$ in the integral operator set $\cbr{\cL^{x,a}:x\in\cX,a\in\cA}$ could be ``dominated'' by some polynomially converging series. The rate parameter $\gamma$ is a key factor in our analysis. In \Cref{sec:functional_regression}, we will prove that parameter $\gamma$ influences the error rate of the functional regression oracle and therefore determines our utility regret rate.

\section{Functional Regression and Oracle Inequality}\label{sec:functional_regression}

\subsection{Functional Regression}\label{subsec:function_reg_alg}
In \Cref{subsec:function_reg_alg}, we provide a method about how to estimate our coefficient function $\theta^*(w)$, given a dataset $\cD=\cbr{(x_j,a_j,y_j)}_{j=1}^{n}$. Here, $y_j$ is sampled according to CDF $F^*_{x_j,a_j}$. 

The intuition of functional regression oracle is to use the operator's spectral decomposition. In ordinary least square regression $\min||\bY-\bX\theta||^2$, the normal matrix $\bX^T\bX$ plays a crucial role ~\citep{goldberger1991course}. Here, we require an operator that behaves similarly.
\begin{definition}[Design Integral Operator]
    For any given dataset $\cD=\cbr{(x_j,a_j,y_j)}_{j=1}^{n}$ whose $|\cD|=n$, we define the design integral operator $\cU_{\cD}$ as
    \[
    \cU_{\cD}(\theta)(w):=\sum_{j=1}^{n}\cL^{x_i,a_i}(\theta)(w)=\sum_{j=1}^{n}\int_{\Omega}\theta(r)\cK^{x_j,a_j}(w,r)d\nu(r).
    \]
    Moreover, we denote the spectral decomposition of $\cU_{\cD}$ as $\sum_{i=1}^{\infty}\lambda_i(\cU_{\cD})e^i_{\cU_{\cD}}$.
\end{definition}
In \cref{sec:computation}, we present a concrete method to compute the eigenvalue sequence $\cbr{\lambda_i}$.

Using \Cref{thm:math_property_ourintegral_operator}, one could verify that $\cU_{\cD}$ is also a linear, positive, self-adjoint, and Hilbert-Schmidt integral operator. Therefore, we have $\lambda_i(\cU_{\cD})\rightarrow0$, and $\cU_{\cD}$ is not invertible. 

In finite-dimensional regression, people often handle non-invertibility with regularization terms $\lambda||\theta||^2$, leading to an error that scales with the problem dimension. However, since our dimension is infinite, adding this regularization term is infeasible. \citet{zhang2022functional} considers a different data-driven regularization, but the term will scale with the cardinality $|\cD|=n$, thus it's inapplicable to obtain convergence results.

Rather than regularization, we consider a truncation of the function series $\sum_{i=1}^{\infty}\lambda_i(\cU_{\cD})e^i_{\cU_{\cD}}$ according to the eigenvalue sequence $\cbr{\lambda_i(\cU_{\cD})}$. Specifically, we define the truncated finite rank operator $\hat{\cU}_{\cD,\varepsilon}$.
\begin{definition}\label{def:hat{cU}}
    For any small number $\varepsilon>0$, the truncated integral operator $\hat{\cU}_{\cD,\varepsilon}$ is defined as 
    $$\hat{\cU}_{\cD,\varepsilon}:\theta\mapsto\sum_{i=1}^{N_{\cD,\varepsilon}}\lambda_i(\cU_{\cD})\inner{\theta}{e^i_{\cU_{\cD}}}e^i_{\cU_{\cD}},$$
    where $N_{\cD,\varepsilon}$ is the smallest number such that for
    $\forall i\ge N_{\cD,\varepsilon}+1,\ \lambda_i(\cU_{\cD})<n\varepsilon.$
\end{definition}
$\varepsilon$ is some hyper-parameter which will be determined later.
Since we arrange the eigenvalues in a decreasing order, we have that for 
$\forall i\le N,\ \lambda_i(\cU_{\cD})\ge n\varepsilon.$ Despite the fact that $\hat{\cU}_{\cD,\varepsilon}$ is not invertible, it is still possible to define its pseudo-inverse as follows.
\begin{definition}\label{def:pseudo_inverse_hat{cU}}
    The pseudo-inverse of the operator $\hat{\cU}_{\cD,\varepsilon}$ is defined as
    \[\hat{\cU}_{\cD,\varepsilon}^{\dagger}:\theta\mapsto\sum_{i=1}^{N_{\cD,\varepsilon}}\frac{1}{\lambda_i(\cU_{\cD})}\inner{\theta}{e^i_{\cU_{\cD}}}e^i_{\cU_{\cD}}.\]
\end{definition}
We set $\varepsilon=\varepsilon^*:=n^{-\frac{2}{\gamma+2}}$ to conduct functional regression. With a bit of abuse of notation, we also use $N_{\varepsilon}$ to denote $N_{\cD,\varepsilon}$, and abbreviate $\hat{\cU}_{\cD,\varepsilon^*}$, $\hat{\cU}_{\cD,\varepsilon^*}^{\dagger}$ as $\hat{\cU}_{\cD}$ and $\hat{\cU}_{\cD}^{\dagger}$.

After obtaining $\hat{\cU}_{\cD}$ and $\hat{\cU}_{\cD}^{\dagger}$, there are two remaining steps of our functional regression.

The first step is to compute the following function 
\[
\theta_{\cD}=\cU_{\cD}^{\dagger}\rbr{\int_{S}\sum_{j=1}^{n}\II_{y_j}(t)\phi(x_j,a_j,w,s)dm(s)}.
\]
The motivation for calculating it lies in its ability to solve the following least squares optimization problem presented in \Cref{thm:optimization}. Intuitively, for any data point $(x_i,a_i)$, the underlying targeted distribution function is $F^*(x_i,a_i,s)=\int_{\Omega}\theta^*(w)\phi(x_i,a_i,w,s)d\nu(w)$. 
Normally, to learn the function $\theta^*(w)$, we need to observe the whole function $F^*(x_i,a_i,s)=\int_{\Omega}\theta^*(w)\phi(x_i,a_i,w,s)d\nu(w)$ and solve an inverse problem to obtain $\theta^*$. However, due to our bandit feedback nature, we only have one sample $y_i$ from this distribution. Thus, considering the $L^2$ error between the real empirical counterpart $\II_{y_i}(t)$ and any candidate distribution function $\int_{\Omega}\theta(w)\phi(x_i,a_i,w,s)d\nu(w)$ for all $n$ data points, we obtain the following loss function $l(\theta,\cD)$ in \Cref{thm:optimization}.
\begin{theorem}\label{thm:optimization}
    We define the loss function $l(\theta,\cD)$ on dataset $\cD$ as
    \[l(\theta,\cD):=\sum_{j=1}^{n}\norm{\II_{y_i}(s)-\int_{\Omega}\theta(w)\phi(x_i,a_i,w,s)d\nu(w)}_{L^2(S,m)}^2.\]
    Then, $\theta_{\cD}$ solves the following optimization problem,
    \begin{equation}\label{eq:opt}
    \bP):\ \min_{\theta\in \spa(e_1,\cdots,e_{N_{\varepsilon}})}l(\theta,\cD),
\end{equation}
where $\spa(X)$ denotes the linear subspace spanned by the vector group $X$.
\end{theorem}
After obtaining $\theta_{\cD}$, in order to ensure that our estimated coefficient function satisfies \Cref{ass:regression_model}, we project it onto the set $\cC=\cbr{\theta\ge 0:\int_{\Omega}\theta d\nu=1,||\theta||_{L^2(\Omega,\nu)}\le M}$ in the following norm.
\begin{definition}
    We define the following norm for any positive self-adjoint bounded linear operator $\cU$:
    \[||\theta||_{\cU}:=\sqrt{\inner{\theta}{\cU(\theta)}}.\]
    Correspondingly, we define the projection $\cP_{\cD}$ onto set $\cC$ under the norm $||\cdot||_{\cU_{\cD}}$ as
    \[
    \cP_{\cD}(x):=\argmin_{y\in\cC}||y-x||_{\cU_{\cD}}^2.
    \]
\end{definition}

The second step is to project $\theta_{\cD}$ onto set $\cC$ by $\cP_{\cD}$, i.e., $\hat{\theta}_{\cD}:=\cP_{\cD}(\theta_{\cD})$.This process yields our estimated coefficient function $\hat{\theta}_{\cD}$.
At the end of this section, we summarize the functional regression oracle $\Alg$ in \Cref{alg:regression_oracle_pseudo_code}.
\begin{algorithm}
\caption{$\Alg$}\label{alg:regression_oracle_pseudo_code}
    \begin{algorithmic}
        \Require Basis function family $\cbr{\phi(x,a,w,t)}_{w\in\Omega}$, dataset $\cD=\cbr{x_i,a_i,y_i}_{i=1}^{n}$, $\cbr{\tau_i}_{i=1}^{\infty}$, $0<\gamma\le1$.
    \State Define the integral operator $\cU_{\cD}(\theta)=\sum_{j=1}^{n}\cL^{x_i,a_i}(\theta)(w)=\sum_{j=1}^{n}\int_{\Omega}\cK^{x_j,a_j}(w,r)\theta(r)d\nu(r)$.
    \State Compute the spectral decomposition of $\cU_{\cD}$, say
    $$\cU_{\cD}(\theta)=\sum_{i=1}^{\infty}\lambda_i\inner{\theta}{e_i}e_i.$$
    \State Define $\varepsilon^*=n^{-\frac{\gamma}{\gamma+2}}$ and construct the truncated operator $\hat{\cU}_{\cD}$ and its pseudo-inverse $\hat{\cU}_{\cD}^{\dagger}$ according to \Cref{def:hat{cU}} and \Cref{def:pseudo_inverse_hat{cU}}.
    \State Solve the optimization problem ~\ref{eq:opt} and obtain solution $\theta_{\cD}$.
    \State Compute the projection $\hat{\theta}_{\cD}=\cP_{\cD}(\theta_{\cD})$.
    \State \Return $\hat{\theta}_{\cD}$.
    \end{algorithmic}
\end{algorithm}

\subsection{Oracle Inequality}\label{subsec:oracle_ineq}

In this section, we give two versions of oracle inequalities. The first one is for the fixed design that $\cD$ is arbitrarily given, whereas the second one is for the random design that $\cbr{(x_i,a_i)}_{i=1}^{n}$ are sampled from some unknown joint distribution $Q$.
,\begin{theorem}\label{thm:reg_fixed_design}
    Under \Cref{ass:regression_model} and \Cref{ass:dominating_eigendecay}, given dataset $\cD=\cbr{(x_j,a_j,y_j)}_{j=1}^{n}$ and setting $\varepsilon^*=n^{-\frac{2}{\gamma+2}}$, then with probability at least $1-\delta$, the estimated coefficient function $\hat{\theta}_{\cD}$ from \Cref{alg:regression_oracle_pseudo_code} satisfies
    \[
    ||\hat{\theta}_{\cD}-\theta^*||_{\cU_{\cD}}\le\cE_{\cD}^{\delta}(n):=\sqrt{2\rbr{2\log(\frac
{1}{\delta})+\sum_{i=1}^{N_{\varepsilon^*}}\log(1+\lambda_i(\cU_{\cD}))}}+n^{\frac{\gamma}{\gamma+2}} M.
\]

\end{theorem}
\Cref{thm:reg_fixed_design} provides a theoretical guarantee about our regression error when the given $\cD$ has cardinality $|\cD|=n$. Nevertheless, \Cref{thm:reg_fixed_design} cannot be directly used, as there exists an agnostic term $\sum_{i=1}^{N_{\varepsilon^*}}\log(1+\lambda_i(\cU_{\cD}))$ related to the dataset $\cD$. The following important lemma gives us an upper bound of $\cE_{\cD}^{\delta}(n)$ which is independent of $\cD$.
\begin{lemma}
\label{lemma:bound_cE_D}
    For any dataset $\cD=\cbr{(x_i,a_i,y_i)}_{i=1}^{n}$ whose $|\cD|=n$, by choosing $\varepsilon=\varepsilon^*=n^{\frac{-2}{\gamma+2}}$, we have
    \begin{align*}
        &\cE_{\cD}^{\delta}(n)=\sqrt{2\rbr{2\log(\frac
{1}{\delta})+\sum_{i=1}^{N_{\varepsilon^*}}\log(1+\lambda_i)}}+n^{\frac{\gamma}{\gamma+2}} M\\
\le& \cE_{\delta}(n):=2\log(1/\delta)^{1/2}+\rbr{2\sqrt{s_0\log(1+n)}+M}n^{\frac{\gamma}{\gamma+2}}.
    \end{align*}
\end{lemma}
With \Cref{lemma:bound_cE_D}, we derive the following corollary which is independent of $\cD$.

\begin{corollary}\label{cor:reg_fixed_design}
    Under \Cref{ass:regression_model} and \Cref{ass:dominating_eigendecay}, given dataset $\cD=\cbr{(x_j,a_j,y_j)}_{j=1}^{n}$ and setting $\varepsilon^*=n^{-\frac{2}{\gamma+2}}$, then with probability at least $1-\delta$, we have
    $||\hat{\theta}_{\cD}-\theta^*||_{\cU_{\cD}}\le \cE_{\delta}(n)$.
\end{corollary}
When the data points are sampled i.i.d. according to some distribution, we could further give an oracle inequality for random design. Taking it into consideration, our oracle inequality for random design helps to bound the $L^2$ error between the ground truth CDF and our estimated CDF.

We first state some mild assumptions in random design.
\begin{assumption}\label{ass:nonsingular}
    We assume that
    for any $x,a,w,r$, $\exists \eta>0$ it holds that that $\cK^{x,a}(w,r)=\int_{S}\phi(x,a,w,s)\phi(x,a,r,s)dm(s)\ge \eta>0$.
\end{assumption}
\Cref{ass:nonsingular} plays a role guarantee the integral operator will not map a function far from $0\in L^2(\Omega,\nu)$ to some near-zero point, which is a bit similar to the concept of proper mapping in convex analysis~\citep{magaril2003convex}.
With \Cref{ass:regression_model}, we know that $\Omega\times\Omega$ is also compact in $\RR^{2d}$. Thus, we have $\cN(t,\Omega\times\Omega,||\cdot||_{\infty})\le \rbr{\frac{A}{t}}^{2d}$ for some $A$, delayed to \Cref{lemma:covering_num}.
We finally transfer the estimation error of $\hat{\theta}_{\cD}$ to the $L^2$ error between cumulative distribution functions,
\begin{align*}
    ||\hat{\theta}_{\cD}-\theta^*||_{\cU_{\cD}}^2=&\inner{\hat{\theta}_{\cD}-\theta^*}{\cU_{\cD}\rbr{\hat{\theta}_{\cD}-\theta^*}}\\
    =&\sum_{j=1}^{n}\inner{\hat{\theta}_{\cD}-\theta^*}{\cL^{x_j,a_j}\rbr{\hat{\theta}_{\cD}-\theta^*}}\\
    =&\sum_{j=1}^{n}\int_{\cX,\cA}\int_{S}\rbr{\int_{\Omega}\rbr{\hat{\theta}_{\cD}-\theta^*}(w)\phi(x_j,a_j,w,s)d\nu(w)}\rbr{\int_{\Omega}\rbr{\hat{\theta}_{\cD}-\theta^*}(r)\phi(x_j,a_j,r,s)d\nu(r)}dm(s)\\
    =&\sum_{j=1}^{n}\sbr{||\hat{F}_{\cD}(x_j,a_j,s)-F^*(x_j,a_j,s)||_{L^2(S,m)}^2}.
\end{align*}
Consequently, combining \Cref{thm:reg_fixed_design} with some concentration analysis and \Cref{lemma:bound_cE_D}, we have the following theorem.
\begin{theorem}\label{thm:reg_random_design}
    Given some dataset $\cD=\cbr{(x_j,a_j,y_j)}_{j=1}^{n}$ where $(x_j,y_j)\overset{i.i.d.}{\sim}Q$, we define $\hat{F}_{\cD}(x,a,s):=\int_{\Omega}\hat{\theta}_{\cD}(w)\phi(x,a,w,s)d\nu(w)$ to be our estimated cumulative distribution function under $x,a$. For any $0\le\delta\le 1/2$, we have that with probability at least $1-2\delta$,
    \[
    \EE_{(x,a)\sim Q}\sbr{||\hat{F}_{\cD}(x,a,s)-F^*(x,a,s)||_{L^2(S,m)}^2}\le \frac{C(d,L_0,\delta,A,\eta)\cE_{\delta}(n)^2}{n},
    \]
    where $C(d,L_0,\delta,A,\eta)=\rbr{1+(48\sqrt{d\log(2L_0A)}+2\sqrt{\log(1/\delta)})/\eta}$ is some constant that does not influence the $L^2$ error rate with respect to $n$.
\end{theorem}

\section{Algorithm}\label{sec:alg}

We provide an algorithm in \Cref{sec:alg} that incorporates our functional regression oracle to minimize the expected utility regret. 

It is computationally expensive to calculate eigenvalues and corresponding eigenfunctions of an integral operator. Hence, it is desirable to develop algorithms with a low-frequency oracle call property. Inspired by \citet{simchi2020bypassing,qian2024offline}, we design an Inverse Gap Weighting policy in a batched version possessing the low-frequency oracle call property.

Before delving into the specific algorithm, the following corollary indicates that a well-bounded estimation error implies a well-bounded decision-making error. 
\begin{theorem}\label{thm:utility_bound}
    Given some dataset $\cD=\cbr{(x_j,a_j,y_j)}_{j=1}^{n}$ where $(x_j,y_j)\overset{i.i.d.}{\sim}Q$, we define $\hat{F}_{\cD}(x,a,s):=\int_{\Omega}\hat{\theta}_{\cD}(w)\phi(x,a,w,s)d\nu(w)$ as our estimated cumulative distribution function. For any $0<\delta< 1$, we have that with probability at least $1-\delta$,
    \[
    \EE_{(x,a)\sim Q}\sbr{\rbr{\cT(\hat{F}_{\cD}(x,a,t))-\cT(F^*(x,a,t))}^2}\le \frac{L^2C(d,L_0,\delta/2,A,\eta)\cE_{\delta/2}(n)^2}{n},
    \]
    where $C(d,L_0,\delta/2,A,\eta)=\rbr{1+(48\sqrt{d\log(2L_0A)}+2\sqrt{\log(2/\delta)})/\eta}$.
\end{theorem}

We use $\mathsf{Est}_{\delta}(n)$ to denote the number $L^2C(d,L_0,\delta,A,\eta)\cE_{\delta}(n)^2$.
For any $x,a$, we could view $\cT(F^*(x,a,s))$ as the unknown expected ``reward'' related to action $a$ given context $x$. Therefore, we transform our abstract sequential decision-making problem into a stochastic contextual bandit problem. Although in our problem, at every round, we do not directly observe a sample point from the ``reward'' distribution, we could still estimate it by functional regression $\Alg$, yielding the desired reduction from the sequential decision-making problem to a contextual bandit. 


We summarize our algorithm in~\Cref{alg:IGW_pseudo_code}.
For the algorithm structure, we first divide the whole $T$ rounds into several epochs and geometrically increase the length of every epoch so that the low-frequency oracle call property is automatically satisfied in \Cref{alg:IGW_pseudo_code}. At the beginning of every epoch $m$, we call our functional regression oracle $\Alg$ based on the i.i.d. data gathered from the last epoch to get an estimation $\hat{\theta}_m$. We then design our inverse gap weighting policy $\pi_m$ based on $\hat{\theta}_m$ and execute it throughout this epoch. By such a structure, we could ensure that the data generated throughout this epoch are i.i.d. so that we can use \Cref{thm:reg_random_design} to bound our $L^2$ estimation error, and therefore the utility regret.

To be specific, we first impose epoch schedule $\tau_m=2^m$, which means that the $m+1$th epoch is twice as long as the $m$th epoch. Therefore, the statistical guarantee we get from epoch $m+1$ is stronger than that of epoch $m$. As $m$ scales, our estimation becomes more and more accurate. The Inverse Gap Weighting policy enables us to balance the exploration and exploitation trade-off to maintain a low regret just assuming access to an offline regression oracle with i.i.d. input data. We now provide the theoretical guarantee of our \Cref{alg:IGW_pseudo_code} as follows.

\begin{algorithm}[ht]
\caption{Stochastic Contextual Decision Making with Infinite Functional Regression}\label{alg:IGW_pseudo_code}
    \begin{algorithmic}
        \Require Functional $\cT$, feature space $\cX$, action space $\cA=\cbr{a_1\cdots,a_K}$ , basis function family $\cbr{\phi(x,a,w,u)}_{w\in\Omega}$,  regression space $L^2(\Omega,\nu)$, range space of random variables $(S,m)$.
        \State Initialize epoch schedule $0=\xi_0<\xi_1<\xi_2<\cdots$
        \For{$m=1,2,\cdots$}
        \State Obtain $\hat{\theta}_m\in L^2(\Omega,\nu)$ from oracle $\Alg$ with dataset
        $\cD_{m-1}=\cbr{(x_t,a_t,y_t)}_{t=\tau_{m-2}+1}^{\tau_{m-1}}$,
        $$\hat{\theta}_m=\Alg(\cD_{m-1}).$$
        \State Define exploration parameter $\varsigma_m=\frac{1}{2}\sqrt{\frac{K}{\mathsf{Est}_{\delta/2m^2}(\tau_{m-1}-\tau_{m-2})}}$ (for epoch 1, $\varsigma_1=1$).
        \For{Round $t=\xi_{m-1}+1\cdots,\xi_m$}
        \State Observe $x_t$.
        \State Compute the value of functional for every action $$\hat{v}_m(a)=\cT(\int_{\Omega}\hat{\theta}_{m}\phi(x_t,a,w,s)d\nu(w))\ \forall a\in\cA.$$ 
        \State Define $\mathsf{IGW}_{\varsigma_m}$ policy:
        \State 
        \begin{equation*}
  p_t(a) =
    \begin{cases}
      \frac{1}{K+\varsigma_m\rbr{\hat{v}_m(\hat{a}_m)-\hat{v}_m(a)}}& \text{for all $a\neq \hat{a}_m$}\\
      1-\sum_{a\neq \hat{a}_t}p_t(a) & \text{for $a=\hat{a}_m$}.
    \end{cases}       
\end{equation*}        
\State Sample $a_t\sim p_t$ and observe one data point from cumulative distribution function $F^*(x_t,a_t,s).$
\EndFor
\EndFor
    \end{algorithmic}
\end{algorithm}

\begin{theorem}\label{thm:theoretical_guarantee} For stochastic context setting, assuming that we can only call functional regression oracle $\ceil{\log(T)}$ times, then with probability at least $1-\delta$, the expected regret of \Cref{alg:IGW_pseudo_code} after $T$ rounds is at most
\[
\EE[\text{Reg}(T)]\le\Tilde{C}(K,L,L_0,A,d,\eta)\rbr{(s_0^{1/2}+M)\vee 1}\cO\rbr{\log(\log T/\delta)\rbr{\sqrt{T}+T^{\frac{3\gamma+2}{2(\gamma+2)}}}},
\]
    where $\Tilde{C}(A,K,L,L_0,d,\eta)$ is some constant that is only relevant to the parameters in the bracket.
\end{theorem}
As $0<\gamma\le 1$, the term $\sqrt{T}$ is dominated by $T^{\frac{3\gamma+2}{2(\gamma+2)}
}$ and our regret rate becomes 
\[\cO\rbr{\Tilde{C}(K,L,L_0,A,d,\eta)\rbr{(s_0^{1/2}+M)\vee 1}T^{\frac{3\gamma+2}{2(\gamma+2)}}(\log(\log T/\delta))}.\]
From \Cref{thm:theoretical_guarantee}, we observe that the regret rate is determined by the decay speed of the eigenvalue sequence. Specifically, when $\gamma\searrow 0$, our regret rate is closer and closer to the $\sqrt{T}$ order, which matches the traditional optimal regret rate. If the eigenvalue sequence is decaying exponentially fast, then it is also decaying polynomially fast for any positive $\gamma>0$. Thus, by choosing $\gamma\searrow 0$, we conclude that the regret rate has order $\tilde{\cO}\rbr{\sqrt{T}}$ for any exponentially decaying eigenvalue sequence. Moreover, for finite-dimensional problems, we could imagine that all the eigenvalues with an index larger than the dimension number are zero. So, we can also set $\gamma=0$ in \Cref{thm:theoretical_guarantee} and recover the optimal regret rate $\tilde{\cO}(T^{1/2})$ up to some logarithmic factors. 

On the other hand, if we have no prior knowledge about the order of the dominating sequence $\cbr{\tau_{i}}_{i=1}^{\infty}$ but just know that it converges, we can simply set it to $1$ and still achieve sublinear regret, as shown in \Cref{cor:gamma=1_regret}. 
\begin{corollary}\label{cor:gamma=1_regret}
    For the dominating sequence $\cbr{\tau_{i}}_{i=1}^{\infty}$, if we only have information that it converges without any knowledge of the order $\gamma$, we can set $\gamma=1$ and get $\frac{3\gamma+2}{2(\gamma+2)}=\frac{5}{6}$. Therefore, we could obtain the following expected regret bound.
    \[
    \EE\text{Reg}(T)\le \cO\rbr{\Tilde{C}(A,K,L,L_0,d,\eta)\rbr{(s_0^{1/2}+M)\vee 1}(\log(\log T/\delta))T^{\frac{5}{6}}}.
    \]
\end{corollary}
In summary, our algorithm is adaptive and robust. It not only recovers the optimal regret rate in finite-dimensional problems and infinite-dimensional problems with exponential decay but can still manage to achieve a sublinear regret with no prior information about the eigendecay as well. These properties demonstrate the versatility of our algorithm and its broad potential for application.

We finish this section with a remark that in the design of \Cref{alg:IGW_pseudo_code}, we do not involve any information about the total round number $T$. If we know $T$ in advance, we could further introduce a more efficient epoch schedule that reduces the offline functional regression oracle call times from $\cO(\log T)$ to $\cO(\log\log T)$. We omit it here and refer to~\citet{simchi2020bypassing,qian2024offline}.

\section{Conclusion and Discussion}\label{sec:conclusion}

In this paper, we establish a general framework for stochastic contextual online decision-making with infinite-dimensional functional regression, which incorporates any application examples with Lipschitz continuous objective functional. We study the relationship between the utility regret and the eigenvalue sequence decay of our design integral operator. Compared with finite-dimensional linear bandits, this connection is new and crucial in functional regression with infinite dimensions. Furthermore, we design a computationally efficient algorithm to solve our sequential decision-making problem based on a novel infinite-dimensional functional regression oracle.

Finally, we would like to discuss some interesting future research directions.
\paragraph{Extension to adversarial contextual decision-making:}  The epoch of \Cref{alg:IGW_pseudo_code} is designed for lower computation cost, through the low-frequency algorithm call property, of computing eigenvalues and eigenfunctions in our functional regression oracle $\Alg$. We observe that the guarantee provided by our oracle in \Cref{thm:reg_fixed_design} inherently adapts to any dataset where the contexts are arbitrarily drawn. This raises a natural question: can we design an efficient online functional regression oracle~\citep{foster2020beyond}, building on \Cref{alg:regression_oracle_pseudo_code}, to address adversarial contextual decision-making problems?

\paragraph{Minimax lower bound:} In our paper, the utility regret rate is $\tilde{\cO}(T^{\frac{3\gamma+2}{2(\gamma+2)}})$. An interesting open question is whether this dependency on $\gamma$ is optimal.
Our guess is that it may not be minimax optimal, and it is worth future exploration into the minimax lower bound of the utility regret with respect to the eigenvalue decay rate $\gamma$? 

\paragraph{Extension to nonlinear functional regression:} In our problem, the relationship between $F^*(x,a,s)$ and $\Phi=\cbr{\phi(x,a,w,s)}_{w\in\Omega}$ is linear because of the linearity of integration. One challenging problem is to extend our methodology to some potential nonlinearity between $F^*(x,a,s)$ and our basis function family and design efficient decision-making algorithms.

We leave these interesting questions as
potential next steps.

\appendix
\bibliographystyle{plainnat}
\bibliography{haichen/sections/refs}

\section{Useful Math Theorems}\label{sec:tech_tools}
\begin{theorem}[Theorem 8.2 in \citet{gohberg2012traces}]
    Let positive operator $\cU(\theta)(w)=\int_{\Omega}\cK(w,r)\theta(r)d\nu(r)$ be self-adjoint. If the kernel $K(w,r)$ is continuous and satisfies the Lipschitz condition,
$$|K(w,r_1)-K(w,r_2)|\le C||s_2-s_1||,$$
then
$$\sum_{j=1}^{\infty}\lambda_j(\cU)<\infty$$
i.e., $\cU$ is a trace class operator.
\end{theorem}

\begin{theorem}[Theorem 8.1 in~\citet{gohberg2012traces}]
    If a positive Hilbert-Schmidt integral operator $\cU$ is associated with kernel $\cK(\cdot,\cdot)$. Then, if $\cU$ is traceable, it holds that $tr(\cU)=\int \cK(x,x)dx<\infty$.
\end{theorem}
\begin{theorem}[Theorem 2.7 in~\citet{ferreira2013eigenvalue}]\label{general_converge_rate_thm}
    Assume $\nu$ is a Borel measure. Let $\cK$ be a kernel in $C(\Omega\times\Omega)\cap L^2(\Omega\times\Omega,\nu\times\nu)$ possessing an integrable diagonal. Finally, assume $\cL$ possesses a $L^2(\cX,\nu)$-convergent spectral representation in the form 
    $$\cL(f)=\sum_{i=1}^{\infty}\lambda_i(\cL)\inner{f}{e_i}e_i,\ f\in L^2(\Omega,\nu),$$
    where $\cbr{e_i}$ is an orthonormal subset of $L^2(\Omega,\nu)$ and the sequence $\cbr{\lambda_i(\cL)}$ is a subset of a circle sector from the origin of $\CC$ with central angle less than $\pi$. Then, the following statements hold:
    
    (i) There exists $\xi\in[0,2\pi]$ and $l>0$ such that
    $$\sum_{i=1}^{\infty}|\lambda_i(\cL)|\le(1+l^2)^{1/2}\int_{\Omega}Re(e^{i\xi}K(x,x))d\nu(x);$$
    (ii) If the eigenvalues of $\cL$ are arranged such that $|\lambda_i(\cL)|\ge |\lambda_{i+1}(\cL)|,i=1,2,\cdots,$ then
    $$|\lambda_i(\cL)|\le\frac{(1+l^2)^{1/2}}{n}\int_{\Omega}Re(e^{i\xi}K(x,x))d\nu(x);$$
    (iii) as $n\rightarrow\infty$,
    $$\lambda_n(\cL)=o(n^{-1}).$$
    (iv) The operator $\cL$ is trace-class. 
\end{theorem}

\begin{definition}[Functional Determinants]
    For any traceable Hilbert-Schmidt integral operator $\cU$, we define the following number as the functional determinants of $\cU$,
    $$\cE_{\cU}:=\sum_{i=1}^{\infty}\log(1+\lambda_i(\cU)).$$ 
    Given dataset $\cD$, we will also use the terminology $\cE_{\cD}$ to denote $\cE_{\cU_{\cD}}$ when it's clear from context. This number depicts the speed of growth of the eigenvalues of our integral operator. When the dimension is finite, one can show that the order is $\cE_{\cD}\simeq \log(|\cD|)$.
\end{definition}
\begin{theorem}[Doob's Maximal Martingale Inequality]\label{Doob_inequality}
Suppose $\cbr{X_k}_{k\ge 0}$ is a sub-martingale with $X_k\ge 0$ almost surely. Then for all $a>0$, we have,
$$\PP(\max_{1\le i\le k}X_i\ge a)\le \frac{\EE[X_k]}{a}.$$
\end{theorem}
\begin{lemma}\label{Gaussian_integral}
    For any $a,b,c\in \RR$, $a>0$, we have
    $$\int_{\RR}\exp\cbr{-\frac{ax^2+bx+c}{2}}dx=\sqrt{\frac{2\pi}{a}}\exp\rbr{\frac{b^2-4ac}{8a}}.$$
\end{lemma}
\begin{proof}[Proof of \Cref{Gaussian_integral}] 
The proof is quite direct. By calculus, we have
\[
-\frac{ax^2+bx+c}{2}=-\frac{a}{2}\rbr{x^2+\frac{b}{a}x+\frac{c}{a}}=-\frac{a}{2}\rbr{x+\frac{b}{2a}}^2+\frac{b^2-4ac}{8a}.
\]
Using the fact that $\int_{\RR}e^{-x^2}=\sqrt{\pi}$, by change of variable $u=\sqrt{\frac{a}{2}}\rbr{x+\frac{b}{2a}}$, just plug it in and we shall finish the proof.
\end{proof}
\begin{theorem}[Courant-Fischer minimax theorem in~\citet{teschl2014mathematical}]\label{Courant-Fischer}

     Let $A$ be a compact, self-adjoint operator on a Hilbert space $H$. Its eigenvalues are listed in decreasing order $$\lambda_1\ge\lambda_2\ge\cdots\ge\lambda_k\ge\cdots.$$ Let $S_k\subset H$ denote a $k$-dimensional subspace. Then,
    $$\lambda_k=\max_{S_k}\min_{\alpha \in S_k,||\alpha||=1}\inner{x}{Ax}.$$
\end{theorem}
\begin{theorem}[Sion's Minimax Theorem in~\citet{sion1958general}]\label{sion_minimax}
    Let $\cX$ and $\cY$ be convex sets in linear topological spaces, and assume $\cX$ is compact. Let $f:\cX\times\cY\rightarrow\RR$ be such that
    i) $f(x,\cdot)$ is concave and upper semicontinuous over $\cY$ for all $x\in\cX$ and ii)$f(\cdot,y)$ is convex and lower semicontinuous over $\cX$ for all $y\in\cY$. Then,
    $$\inf_{x\in\cX}\sup_{y\in\cY}f(x,y)=\sup_{y\in\cY}\inf_{x\in\cX}f(x,y).$$
\end{theorem}
\begin{lemma}[Example 27.1~\citet{shalev2014understanding}]\label{lemma:covering_num}
    Suppose that $A$ is a compact set in $\RR^m$, and set $c=\max_{a\in A}||a||$, then $\cN(r,A,||\cdot||)\le\rbr{\frac{2c\sqrt{m}}{r}}^{m}$
\end{lemma}

\section{Proofs in \cref{sec:model}}\label{sec:proofs_model}
\begin{proof}[Proof of \Cref{thm:math_property_ourintegral_operator}]
    Recall that
    $$\cL^{x,a}(\theta)(w)=\int_{\Omega}\cK^{x,a}(w,r)\theta(r)d\nu(r).$$
    By definition, we have $0\le\cK^{x,a}(w,r)=\cK^{x,a}(r,w)\le 1$, and
    $$\int_{\Omega}\int_{\Omega}\cK^{x,a}(w,r)^2\le 1.$$
    Thus,$\cL^{x,a}$ is a Hilbert-Schmidt integral operator and it is compact.
    By Cauchy-Schwarz inequality, we have
    $$||\cL^{x,a}(\theta)||_{L^2(\Omega)}^2\le \int_{\Omega}\int_{\Omega}\theta(r)^2d\nu(r)\int_{\Omega}\cK^{x,a}(w,r)^2d\nu(r)d\nu(w)\le ||\theta||_{L^2(\Omega)}.$$
    Also,
    $$\inner{\theta}{\cL^{x,a}(\theta)}=\int_{\Omega}\int_{\Omega}\theta(w)\cK^{x,a}(w,r)\theta(r)d\nu(r)d\nu(w)\ge 0.$$
    Therefore, $\cL^{x,a}$ is positive operator with $||\cL^{x,a}||\le 1$.
    By direct calculation, we have 
    \begin{align*}
        \inner{\theta_1}{\cL^{x,a}\theta_2}=&\int_{\Omega}\theta_1(w)\int_{\Omega}\cK^{x,a}(w,r)\theta_2(r)d\nu(r)d\nu(w)\\
        =&\int_{\Omega}\theta_2(r)\int_{\Omega}\cK^{x,a}(w,r)\theta_1(w)d\nu(w)d\nu(r)\\
        =&\inner{\cL^{x,a}(\theta_1)}{\theta_2}.
    \end{align*}
    So we finish the proof.
    \end{proof}

\section{Proofs in \Cref{sec:functional_regression}}\label{sec:proofs_regression}

\begin{proof}[Proof of \Cref{thm:optimization}]
    
    We first show that $\sum_{j=1}^{n}\int_{S}\II_{y_j}(s)\phi(x_j,a_j,w,s)dm(s)\in L^2(\Omega,\nu)$. We have
    \begin{align*}
        ||\sum_{j=1}^{n}\int_{S}\II_{y_j}(s)\phi(x_j,a_j,w,s)dm(s)||_{L^2(\Omega,\nu)}^2&\le \sum_{j=1}^{n} ||\int_{S}\II_{y_j}(s)\phi(x_j,a_j,w,s)dm(s)||_{L^2(\Omega,\nu)}^2\\
        &\le n\nu(\Omega)<\infty.
    \end{align*}
    Now we define 
    $$\theta_0(w):=\cU_{\cD,\varepsilon}^{\dagger}\rbr{\int_{S}\sum_{j=1}^{n}\II_{y_j}(s)\phi(x_j,a_j,w,s)dm(s)}.$$
Because there are only finite terms in the summation by the definition of $\cU_{\cD,\varepsilon}^{\dagger}$, we have
$$||\theta_0||_{L^2(\Omega,\nu)}<\infty.$$
    Therefore, by the Cauchy-Schwarz inequality, we have
    $$|\int_{\Omega}\theta_0(w)\phi(x_j,a_j,w,s)d\nu(w)|\le \int_{\Omega}|\theta_0(w)\phi(x_j,a_j,w,s)|d\nu(w)\le ||\theta_0||_{L^2(\Omega,\nu)}||\phi(x_j,a_j,\cdot,t)||_{L^2(\Omega,\nu)}\le ||\theta_0||_{L^2(\Omega,\nu)}\sqrt{\nu(\Omega)}.$$
    Recall we have $m(S)=1<\infty$ and $|\II_{y_j}(s)|\le 1$ for any $j\in[n]$ and $t\in S$, which implies that
    $$L(\theta_0,\cD)<\infty.$$
    On the other hand, for any $\theta\in \spa(e_1\cdots,e_{N_{\varepsilon}})$, we have
    \begin{align*}
        L(\theta_0+\theta,\cD)-L(\theta_0,\cD)&= \sum_{j=1}^{n}\int_{S}\rbr{\int_{\Omega}\theta(w)\phi(x_j,a_j,w,s)}^2dm(s)\\
        &+2\sum_{j=1}^{n}\int_{S}\rbr{\int_{\Omega}\theta(w)\phi(x)j,a_j,w,t)d\nu(w)}\rbr{\int_{\Omega}\theta_0(w)\phi(x_j,a_j,w,s)d\nu(w)-\II_{y_j}(s)}dm(s)
    \end{align*}
    We analyze this difference term by term. First, by algebra, it holds that
    \begin{align*}
        &\sum_{j=1}^{n}\int_{S}\rbr{\int_{\Omega}\theta(w)\phi(x_j,a_j,w,s)d\nu(w)}\rbr{\int_{\Omega}\theta_0(w)\phi(x_j,a_j,w,s)d\nu(w)}dm(s)\\
        =&\int_{\Omega}\theta(w)\int_{\Omega}\theta_0(r)\rbr{\sum_{j=1}^{n}\int_{S}\phi(x_j,a_j,w,s)\phi(x_j,a_j,r,s)dm(s)}d\nu(w)d\nu(r)\\
        =&\int_{\Omega}\theta(w)\sum_{j=1}^{n}\cL^{x_j,a_j}(\theta_0)(w)d\nu(w)\\
        =&\inner{\theta}{\cU_{\cD}(\theta_0)}.
    \end{align*}
    Recall $\theta,\theta_0\in\spa(e_1\cdots,e_{N_{\varepsilon}})$, so we have
    $$\inner{\theta}{\cU_{\cD}(\theta_0)}=\inner{\theta}{\hat{\cU}_{\cD,\varepsilon}(\theta_0)}.$$
    Second, by direct calculation, we also have
    \begin{align*}
        &\sum_{j=1}^{n}\int_{S}\rbr{\int_{\Omega}\theta(w)\phi(x)j,a_j,w,t)d\nu(w)}\II_{y_j}(s)dm(s)\\
        =&\int_{\Omega}\theta(w)\rbr{\sum_{j=1}^{n}\int_{S}\II_{y_j}(s)\phi(x_j,a_j,w,s)dm(s)}d\nu(w)\\
        =&\inner{\theta}{\sum_{j=1}^{n}\int_{S}\II_{y_j}(s)\phi(x_j,a_j,\cdot,t)dm(s)}
    \end{align*}
   Given the condition that $\theta\in\spa(e_1\cdots,e_{N_{\varepsilon}})$, we have
    \begin{align*}
        \inner{\theta}{\sum_{j=1}^{n}\int_{S}\II_{y_j}(s)\phi(x_j,a_j,\cdot,s)dm(s)}=\sum_{i=1}^{N_{\varepsilon}}\inner{\theta}{e_i}\inner{\sum_{j=1}^{n}\int_{S}\II_{y_j}(s)\phi(x_j,a_j,\cdot,s)dm(s)}{e_i}
    \end{align*}
    
    Moreover, by the definition of $\theta_0$, we know that
    \begin{align*}
        \inner{\theta}{\hat{\cU}_{\cD,\varepsilon}(\theta_0)}&=\sum_{i=1}^{N_{\varepsilon}}\inner{\theta}{e_i}\inner{\lambda_i\cdot\frac{1}{\lambda_i}\sum_{j=1}^{n}\int_{S}\II_{y_j}(s)\phi(x_j,a_j,\cdot,s)dm(s)}{e_i}\\
        &=\sum_{i=1}^{N_{\varepsilon}}\inner{\theta}{e_i}\inner{\sum_{j=1}^{n}\int_{S}\II_{y_j}(s)\phi(x_j,a_j,\cdot,s)dm(s)}{e_i}.
    \end{align*}
    
    Thus,
    $$L(\theta_0+\theta,\cD)-L(\theta_0,\cD)=\sum_{j=1}^{n}\int_{S}\rbr{\int_{\Omega}\theta(w)\phi(x_j,a_j,w,s)}^2dm(s)\ge 0.$$
    Therefore, $\theta_0=\theta_{\cD,\varepsilon}$ and we finish the proof.
\end{proof}

\begin{lemma}\label{lemma:proj_nonexpansive}
Projection $\cP_{\cD}$ is non-expansive, which implies that
    $$||\cP_{\cD}(x)-\cP_{\cD}(z)||_{\cU_{\cD}}\le ||x-z||_{\cU_{\cD}},\ \forall x,z\in \cL^2(\Omega,\nu).$$
\end{lemma}
\begin{proof}[Proof of \Cref{lemma:proj_nonexpansive}]
    By the fact that $\cC$ is convex, we define the quadratic function $h(\mu)=||x-\rbr{(1-\mu)\cP_{\cD}(x)+\mu y}||_{\cU_{\cD}}^2$ for any $x\in L^2(\Omega,\nu),y\in\cC$. We expand it to get,
    $$h(\mu)=||\rbr{x-\cP_{\cD}(x)}-\mu(y-\cP_{\cD}(x))||_{\cU_{\cD}}^2=||x-\cP_{\cD}(x)||_{\cU_{\cD}}^2-2\mu\inner{x-\cP_{\cD}(x)}{\cU_{\cD}(y-\cP_{\cD}(x))}+\mu^2||y-\cP_{\cD}(x)||_{\cU_{\cD}}^2.$$
    The vertex of $h(\mu)$ is $$\mu^*=\frac{\inner{x-\cP_{\cD}(x)}{\cU_{\cD}(y-\cP_{\cD}(x))}}{||y-\cP_{\cD}(x)||_{\cU_{\cD}}^2}.$$
    Because $\cP_{\cD}(x)=\argmin_{y\in\cC}||y-x||_{\cU_{\cD}}^2$, we have $\mu^*\le 0$. Thus,
    $$\inner{x-\cP_{\cD}(x)}{\cU_{\cD}(y-\cP_{\cD}(x))}\le 0.$$
Again, we abbreviate $\lambda_i$ as $\lambda_i(\cU_{\cD})$ and $e_i$ as $e^i_{\cU_{\cD}}$.    By the spectral representation of $\cU_{\cD}$, this is equivalent to
    \begin{equation}\label{equa:projection_proof_1}
        \sum_{i=1}^{\infty}\lambda_i\inner{x-\cP_{\cD}(x)}{e_i}\inner{y-\cP_{\cD}(x)}{e_i}\le 0.
    \end{equation}
    This holds for all $y\in \cC$, so we set $y=\cP_{\cD}(z)$ in \Cref{equa:projection_proof_1} to get
    \begin{equation}\label{equa:projection_proof_2}
        \sum_{i=1}^{\infty}\lambda_i\inner{x-\cP_{\cD}(x)}{e_i}\inner{\cP_{\cD}(z)-\cP_{\cD}(x)}{e_i}\le 0.
    \end{equation}
    On the other hand, we switch $x,z$ in \Cref{equa:projection_proof_2} to get,
    \begin{equation}\label{equa:proj_proof_3}
        \sum_{i=1}^{\infty}\lambda_i\inner{\cP_{\cD}(z)-z}{e_i}\inner{\cP_{\cD}(z)-\cP_{\cD}(x)}{e_i}\le 0.
    \end{equation}
    Adding \Cref{equa:proj_proof_3} and \Cref{equa:projection_proof_2} to obtain
    $$\sum_{i=1}^{\infty}\lambda_i\inner{\cP_{\cD}(z)-\cP_{\cD}(x)}{e_i}\inner{\cP_{\cD}(z)-\cP_{\cD}(x)}{e_i}\le \sum_{i=1}^{\infty}\lambda_i\inner{z-x}{e_i}\inner{\cP_{\cD}(z)-\cP_{\cD}(x)}{e_i}.$$
    By Cauchy-Schwartz inequality, we have,
    $$\sum_{i=1}^{\infty}\lambda_i\inner{z-x}{e_i}\inner{\cP_{\cD}(z)-\cP_{\cD}(x)}{e_i}\le \rbr{\sum_{i=1}^{\infty}\lambda_i\inner{z-x}{e_i}^2}^{1/2}\rbr{\sum_{i=1}^{\infty}\lambda_i\inner{\cP_{\cD}(z)-\cP_{\cD}(x)}{e_i}^2}^{1/2}.$$
    Plugging this back, we get
    $$(\sum_{i=1}^{\infty}\lambda_i\inner{\cP_{\cD}(z)-\cP_{\cD}(x)}{e_i}^2)^{1/2}\le \rbr{\sum_{i=1}^{\infty}\lambda_i\inner{z-x}{e_i}^2}^{1/2},$$
    which is exactly $||\cP_{\cD}(x)-\cP_{\cD}(z)||_{\cU_{\cD}}\le ||x-z||_{\cU_{\cD}}$.
So we finish the proof.
\end{proof}

\begin{proof}[Proof of \Cref{thm:reg_fixed_design}]
In this section, we first prove for a general $\varepsilon$, and then we replace $\varepsilon$ by $\varepsilon^*=n^{-\frac{2}{\gamma+2}}$ to obtain our result. For the ease of notation, we use $\lambda_i$ to denote $\lambda_i(\cU_{\cD})$ and $e_i$ to denote $e^i_{\cU_{\cD}}$.
We first compute the difference 
\begin{align}\label{equation between two theta}
    &\theta_{\cD,\varepsilon}-\theta^*\nonumber\\
    =&\cU_{\cD,\varepsilon}^{\dagger}\rbr{\int_{S}\sum_{j=1}^{n}\II_{y_j}(s)\phi(x_j,a_j,w,s)dm(s)}-\sum_{i=1}^{\infty}\inner{\theta^*}{e_i}e_i\nonumber\\
    =&\cU_{\cD,\varepsilon}^{\dagger}\rbr{\int_{S}\sum_{j=1}^{n}\II_{y_j}(s)\phi(x_j,a_j,w,s)dm(s)}-\cU_{\cD,\varepsilon}^{\dagger}\rbr{\cU_{\cD}(\theta^*)}-\sum_{i\ge N_{\varepsilon}+1}\inner{\theta^*}{e_i}e_i\nonumber\\
    =&\cU_{\cD,\varepsilon}^{\dagger}\rbr{\int_{S}\sum_{j=1}^{n}\II_{y_j}(s)\phi(x_j,a_j,w,s)dm(s)-\int_{S}\int_{\Omega}\theta^*(r)\phi(x_j,a_j,r,s)d\nu(r)\phi(x_j,a_j,w,s)dm(s)}-\sum_{i\ge N_{\varepsilon}+1}\inner{\theta^*}{e_i}e_i.
\end{align}
We define 
$$V_j(w):=\int_{S}\II_{y_j}(s)\phi(x_j,a_j,w,s)dm(s)-\int_{S}\int_{\Omega}\theta^*(r)\phi(x_j,a_j,r,s)d\nu(r)\phi(x_j,a_j,w,s)dm(s),$$
$$W_n:=\sum_{j=1}^{n}V_j.$$
We apply norm $||\cdot||_{\cU_{\cD}}$ on both sides of \Cref{equation between two theta} yielding
\begin{align}\label{norm_difference between two theta}
    ||\theta_{\cD,\varepsilon}-\theta^*||_{\cU_{\cD}}\le &||\cU_{\cD,\varepsilon}^{\dagger}(W_n)||_{\cU_{\cD}}+||\sum_{i\ge N_{\varepsilon}+1}\inner{\theta^*}{e_i}e_i||_{\cU_{\cD}}\nonn\\
    =&||W_n||_{\hat{\cU}_{\cD,\varepsilon}^{\dagger}}+||\sum_{i\ge N_{\varepsilon}+1}\inner{\theta^*}{e_i}e_i||_{\cU_{\cD}}\nonn\\
    \le&||W_n||_{\hat{\cU}_{\cD,\varepsilon}^{\dagger}}+n\varepsilon M
\end{align}
The equality is due to the definition of $\hat{\cU}_{\cD,\varepsilon}^{\dagger}$.
Hence, our analysis only needs to focus on the first term $||W_n||_{\hat{\cU}_{\cD,\varepsilon}^{\dagger}}$.
We use $\cF_k$ to denote the $\sigma$-algebra generated by all the data from $(x_1,a_1,y_1)$ to $(x_k,a_k,y_k)$, and $\cF_{\infty}$ is $\sigma(\cup_k\cF_k)$ correspondingly.

By the definition of $V_j$, we take expectation with respect to $\cF_{j-1}$ and obtain
$$\EE[V_j(w)|\cF_{j-1}]=\int_{S}\int_{\Omega}\theta^*(r)\phi(x_j,a_j,r,s)d\nu(r)\phi(x_j,a_j,w,s)dm(s)-\theta^*(r)\phi(x_j,a_j,r,s)d\nu(r)\phi(x_j,a_j,w,s)dm(s)=0.$$
Therefore, $\cbr{V_j}_{j\ge 1}$ is a martingale difference sequence and consequently, $W_k=\sum_{j=1}^{k}V_j$ is a martingale. For any $\alpha\in L^2(\Omega,\nu)$, we define $M_k(\alpha):=\exp\cbr{\inner{\alpha}{W_k}-\frac{1}{2}||\alpha||_{\cU_{\cD}}^2}$. Moreover, we define $M_0(\alpha)=1$ for any $\alpha$.
We have the following \Cref{M_k(alpha)martingale} that $\cbr{M_k(\alpha)}_{k\ge0}$ is a non-negative supermartingale.

\begin{lemma}
    \label{M_k(alpha)martingale}$\cbr{M_k(\alpha)}_{k\ge0}$ is a non-negative supermartingale.
\end{lemma}

Define $\zeta=\cbr{\zeta_i}_{i=1}^{\infty}$ as an infinite sequence of i.i.d. standard Gaussian random variables. We use $\cF^{\zeta}$ to denote the $\sigma(\zeta_1,\zeta_2,\cdots)$. Define
$$\beta:=\sum_{i=1}^{N_{\varepsilon}}\zeta_ie_i,$$

Because there are only finite terms in the summation, it holds that
$$\EE[||\beta||_{L^2(\Omega,\nu)}^2]<\infty.$$
Thus, we have $||\beta||_{L^2(\Omega,\nu)}<\infty$ a.s. and $\beta\in L^2(\Omega,\nu)$.

Denoting $\xoverline[0.7]{M}_k$ as $\EE[M_k(\beta)|\cF_{\infty}]$, we now verify that this is also a non-negative supermartingale.
\begin{align*}
    \EE[\xoverline[0.7]{M}_k|\cF_{k-1}]&=\EE[{M}_k(\beta)|\cF_{k-1}]\\
    &=\EE[\EE[{M}_k(\beta)|\cF^{\zeta},\cF_{k-1}]|\cF_{k-1}]\\
    &\le\EE[\EE[{M}_{k-1}(\beta)|\cF^{\zeta},\cF_{k-1}]|\cF_{k-1}]\\
    &=\EE[\EE[{M}_{k-1}(\beta)|\cF^{\zeta},\cF_{\infty}]|\cF_{k-1}]\\
    &=\EE[\EE[{M}_{k-1}(\beta)|\cF_{\infty}]|\cF_{k-1}]\\
    &=\xoverline[0.7]{M}_{k-1}.
\end{align*}
Furthermore, the integral of $|\xoverline[0.7]{M}_n|$ could be upper bounded by,
$$\EE[|\xoverline[0.7]{M}_k|]=\EE[\xoverline[0.7]{M}_n]=\EE[M_k(\beta)]=\EE[\EE[M_k(\beta)|\cF^{\zeta}]]\le \EE[M_0(\beta)]=1.$$

Therefore, $\cbr{\xoverline[0.7]{M}_k}_{n\ge k\ge 0}$ is a non-negative supermartingale with a uniform expectation upper bound. 

On the other hand, we can directly calculate $\xoverline[0.7]{M}_n$. We use $w_i$ to denote $\inner{W_n}{e_i}$ for simplicity. Then, it holds that
$$M_n(\beta)=\exp\cbr{\inner{\beta}{W_n}-\frac{1}{2}||\beta||_{\cU_{\cD}}^2}=\exp\cbr{\sum_{i=1}^{N_{\varepsilon}}w_i\zeta_i-\frac{1}{2}\sum_{i=1}^{N_{\varepsilon}}\lambda_i\zeta_i^2},$$
and
\begin{align*}
    \xoverline[0.7]{M}_n=&\int_{\RR^{N_{\varepsilon}}}\exp\cbr{\sum_{i=1}^{N_{\varepsilon}}w_i\zeta_i-\frac{1}{2}\sum_{i=1}^{N_{\varepsilon}}\rbr{\lambda_i\zeta_i^2+\zeta_i^2}}\frac{1}{\sqrt{2\pi}}d\zeta_1\frac{1}{\sqrt{2\pi}}d\zeta_2\cdots\frac{1}{\sqrt{2\pi}}d\zeta_{N_{\varepsilon}}\\
    =&\prod_{i=1}^{N_{\varepsilon
    }}\rbr{\int_{\RR}\exp\cbr{w_i\zeta_i-\frac{\lambda_i+1}{2}\zeta_i^2}\frac{1}{\sqrt{2\pi}}d\zeta_i}.
\end{align*}

Then, by \Cref{Gaussian_integral}, we have
\begin{align*}
    \xoverline[0.7]{M}_n=\frac{1}{\sqrt{\prod_{i=1}^{N_\varepsilon}(1+\lambda_i)}}\exp\cbr{\sum_{i=1}^{N_{\varepsilon}}\frac{w_i^2}{2(1+\lambda_i)}}\ge \frac{1}{\sqrt{\prod_{i=1}^{N_\varepsilon}(1+\lambda_i)}}\exp\cbr{\frac{1}{2(1+1/(n\varepsilon))}||W_n||_{\cU_{\cD,\varepsilon}^{\dagger}}^2}.
\end{align*}
Applying \Cref{Doob_inequality}, we derive 
$$\PP\rbr{\frac{1}{\sqrt{\prod_{i=1}^{N_\varepsilon}(1+\lambda_i)}}\exp\cbr{\frac{1}{2(1+1/(n\varepsilon))}||W_n||_{\cU_{\cD,\varepsilon}^{\dagger}}^2}\ge\frac{1}{\delta}}\le\delta,$$
which implies that with probability $1-\delta$,
$$||W_n||_{\cU_{\cD,\varepsilon}^{\dagger}}^2< 2(1+1/(n\varepsilon))\sbr{\log(\frac
{1}{\delta})+\frac{1}{2}\sum_{i=1}^{N_{\varepsilon}}\log(1+\lambda_i)}.$$
Combined with \Cref{norm_difference between two theta}, it indicates that with probability $1-\delta$,
$$||\theta_{\cD,\varepsilon}-\theta^*||_{\cU_{\cD}}\le\sqrt{2(1+1/(n\varepsilon))\sbr{\log(\frac
{1}{\delta})+\frac{1}{2}\sum_{i=1}^{N_{\varepsilon}}\log(1+\lambda_i)}}+n\varepsilon M.$$
Noting that projection mapping is non-expansive by \Cref{lemma:proj_nonexpansive}, we have
$$||\cP_{\cD}(\theta_{\cD,\varepsilon})-\cP_{\cD}(\theta^*)||_{\cU_{\cD}}=||\hat{\theta}_{\cD,\varepsilon}-\theta^*||_{\cU_{\cD}}\le ||\theta_{\cD,\varepsilon}-\theta^*||_{\cU_{\cD}}.$$
Finally, we conclude that
$$||\hat{\theta}_{\cD,\varepsilon}-\theta^*||_{\cU_{\cD}}\le\sqrt{(1+1/n\varepsilon)\rbr{2\log(\frac
{1}{\delta})+\sum_{i=1}^{N_{\varepsilon}}\log(1+\lambda_i)}}+n\varepsilon M.$$

We define
$$\cE_{\cD,\varepsilon}^{\delta}(n):=\sqrt{2\rbr{2\log(\frac
{1}{\delta})+\sum_{i=1}^{N_{\varepsilon}}\log(1+\lambda_i)}}+n\varepsilon M.$$
Setting $\varepsilon=\varepsilon^*=n^{-\frac{2}{\gamma+2}}$ and notice that $\frac{1}{n\varepsilon}=n^{-\frac{\gamma}{\gamma+2}}\le 1$,
So we finish the proof.
\end{proof}
\begin{proof}[Proof of \Cref{M_k(alpha)martingale}]
    By definition, we know that $M_k(\alpha)$ is $\cF_{k}$-measurable. Consequently, it holds that
    \begin{align*}
        \EE\sbr{M_k(\alpha)|\cF_{k-1}}M_{k-1}(\alpha)=&\EE\sbr{\exp\cbr{\alpha,V_k}-\frac{1}{2}\int_{S}\psi_{x_k,a_k}(\alpha,t)^2dm(s)|\cF_{k-1}}\\
        =&M_{k-1}(\alpha)\frac{\EE\sbr{\exp\cbr{\alpha,V_k}|\cF_{k-1}}}{\exp\cbr{\frac{1}{2}\int_{S}\psi_{x_k,a_k}(\alpha,t)^2dm(s)}}\\
    \end{align*}
    Since $-\int_{S}|\psi_{x_k,a_k}(\alpha,t)|dm(s)\le\inner{\alpha}{V_k}\le\int_{S}|\psi_{x_k,a_k}(\alpha,t)|dm(s)$, according to Hoeffding's lemma~\citep{hoeffding1994probability}, we have
    \begin{align*}
        \EE\sbr{\exp\cbr{\alpha,V_k}|\cF_{k-1}}\le \exp\cbr{\frac{4}{8}\rbr{\int_{S}|\psi_{x_k,a_k}(\alpha,t)|dm(s)}^2}\le \exp\cbr{\frac{1}{2}\int_{S}\psi_{x_k,a_k}(\alpha,t)^2dm(s)}.
    \end{align*}
    Then we have
    $$\EE\sbr{M_k(\alpha)|\cF_{k-1}}\le M_{k-1}(\alpha)\frac{\exp\cbr{\frac{1}{2}\int_{S}\psi_{x_k,a_k}(\alpha,t)^2dm(s)}}{\exp\cbr{\frac{1}{2}\int_{S}\psi_{x_k,a_k}(\alpha,t)^2dm(s)}}=M_{k-1}(\alpha).$$
    We finish the proof.
\end{proof}
\begin{proof}[Proof of \Cref{lemma:bound_cE_D}]

For any $k$ fixed, by \Cref{Courant-Fischer}, we have,
\begin{align*}
    \lambda_{k}(\cU_{\cD})=&\max_{S_k,\dim(S_k)=k}\min_{\alpha\in S_k,||\alpha||=1}\inner{\alpha}{(\cL^{x_1,a_1}+\cdots+\cL^{x_n,a_n})(\alpha)}\\
    =&\max_{S_k,\dim(S_k)=k}\min_{\alpha\in S_k,||\alpha||=1}\sum_{j=1}^{n}\inner{\alpha}{\cL^{x_j,a_j}(\alpha)}\\
    \le&\max_{S_k,\dim(S_k)=k}\min_{\alpha\in S_k,||\alpha||=1}\max_{\cL\in\cbr{\cL^{x,a}:x\in\cX,a\in\cA}}n\inner{\alpha}{\cL(\alpha)}
\end{align*}
Notice that for any $S_k$ fixed, the set $\cbr{\alpha\in S_k,||\alpha||=1}$ is both compact and convex. By \Cref{ass:dominating_eigendecay}, we know that $\cbr{\cL^{x,a}:x\in\cX,a\in\cA}$ is convex. We regard $\inner{\alpha}{\cL(\alpha)}$ as a function of both $\alpha$ and $\cL$. Then, function $f(\alpha,\cL):=\inner{\alpha}{\cL(\alpha)}$ is linear in $\cL$, and so $f(\alpha,\cL)$ is continuous and concave in $\cL$. On the other hand, $f(\alpha,\cL)$ is the square of some norm $||\cdot||_{\cL}$ and thus it is convex and continuous in $\alpha$. We apply \Cref{sion_minimax} to get
\begin{align*}
    &\max_{S_k,\dim(S_k)=k}\min_{\alpha\in S_k,||\alpha||=1}\max_{\cL\in\cbr{\cL^{x,a}:x\in\cX,a\in\cA}}\inner{\alpha}{\cL(\alpha)}\\
    =&\max_{S_k,\dim(S_k)=k}\max_{\cL\in\cbr{\cL^{x,a}:x\in\cX,a\in\cA}}\min_{\alpha\in S_k,||\alpha||=1}\inner{\alpha}{\cL(\alpha)}\\
    =&\max_{\cL\in\cbr{\cL^{x,a}:x\in\cX,a\in\cA}}\max_{S_k,\dim(S_k)=k}\min_{\alpha\in S_k,||\alpha||=1}\inner{\alpha}{\cL(\alpha)}.
\end{align*}
The last inequality holds because we can swap the order of two supremums.
Using \Cref{Courant-Fischer} once again, we have
$$\max_{\cL\in\cbr{\cL^{x,a}:x\in\cX,a\in\cA}}\max_{S_k,\dim(S_k)=k}\min_{\alpha\in S_k,||\alpha||=1}\inner{\alpha}{\cL(\alpha)}=\max_{\cL\in\cbr{\cL^{x,a}:x\in\cX,a\in\cA}}\lambda_k(\cL).$$
Combining these parts together, it holds that 
$$\lambda_k(\cU_{\cD})\le n\max_{\cL\in\cbr{\cL^{x,a}:x\in\cX,a\in\cA}}\lambda_k(\cL)\le n\tau_i.$$

Thus, we have
\begin{align*}
    \sum_{i=1}^{N_{\varepsilon}}\log(1+\lambda_i(\cU_{\cD}))=&\sum_{i=1}^{N_{\varepsilon}}\log(1+\frac{\lambda_i}{n\varepsilon}n\varepsilon) \le\sum_{i=1}^{N_{\varepsilon}}\log(1+\frac{n\tau_i}{n\varepsilon}n\varepsilon)
\end{align*}
Because for $\forall 1\le i\le N_{\varepsilon}$, it holds that
$$n\varepsilon\le \lambda_i\le n\tau_i,$$
we have
$$\sum_{i=1}^{N\varepsilon}\log(1+\frac{n\tau_i}{n\varepsilon}n\varepsilon)=\sum_{i=1}^{N_{\varepsilon}}\log(1+n\tau_i)\le \sum_{i=1}^{N_{\varepsilon}}\log(1+n\tau_i^{1-\gamma}\varepsilon^\gamma\frac{\tau_i^{\gamma}}{\varepsilon^\gamma})\le \sum_{i=1}^{N_\varepsilon}\frac{\tau_i^{\gamma}}{\varepsilon^\gamma}\log(1+n\tau_i^{1-\gamma}\varepsilon^\gamma)\le\frac{\log(1+\varepsilon^{\gamma} n)}{\varepsilon^{\gamma}}\sum_{i=1}^{N_{\varepsilon}}\tau_i,$$
which implies
\begin{align}\label{bound_sum_log}
    \sum_{i=1}^{N_{\varepsilon}}\log(1+\lambda_i(\cU_{\cD}))\le \frac{\log(1+\varepsilon^\gamma n)}{\varepsilon^\gamma}\sum_{i=1}^{N_{\varepsilon}}\tau_i.
\end{align}
Plugging \Cref{bound_sum_log} back to the definition of $\cE_{\cD,\varepsilon}^{\delta}(n)$, we have,
\begin{align*}
    \cE_{\cD,\varepsilon}^{\delta}(n)=&\sqrt{2\rbr{2\log(\frac
{1}{\delta})+\sum_{i=1}^{N_{\varepsilon}}\log(1+\lambda_i)}}+n\varepsilon M\\
\le& \sqrt{2\rbr{2\log(\frac
{1}{\delta})+\frac{\log(1+\varepsilon^\gamma n)}{\varepsilon^\gamma}\sum_{i=1}^{N_{\varepsilon}}\tau_i}}+n\varepsilon M\\
\le& \sqrt{2\rbr{2\log(\frac
{1}{\delta})+\frac{\log(1+\varepsilon^\gamma n)}{\varepsilon^\gamma}s_0}}+n\varepsilon M.
\end{align*}
Setting $\varepsilon^*=n^{-\frac{2}{\gamma+2}}$, we get
$$\cE_{\cD,\varepsilon^*}^{\delta}(n)\le\sqrt{4\log(1/\delta)+2s_0n^{\frac{2\gamma}{\gamma+2}}\log(1+n)}+n^{\frac{\gamma}{\gamma+2}}M.$$
The RHS has nothing to do with $\cD$ but the information $|\cD|=n$. Therefore, we define $\cE_{\delta}(n):=2\log(1/\delta)^{1/2}+\rbr{2\sqrt{s_0\log(1+n)}+M}n^{\frac{\gamma}{\gamma+2}}$. By the fact that $\sqrt
{a+b}\le \sqrt{a}+\sqrt{b}$, we obtain,
$$\cE_{\cD,\varepsilon^*}^{\delta}(n)\le \cE_{\delta}(n)=2\log(1/\delta)^{1/2}+\rbr{2\sqrt{s_0\log(1+n)}+M}n^{\frac{\gamma}{\gamma+2}}.$$
Thus, we finish the proof.
\end{proof}
\begin{proof}[Proof of \Cref{thm:reg_random_design}]
Denoting $\cK^{x,a}(w,r)$ as $g_{w,r}(x,a)$, we get a function family indexed by $\Omega\times\Omega$, $\cG=\cbr{g_{w,r}(x,a):(w,r)\in\Omega\times\Omega}$.
We first consider all the rational points in $\Omega\times\Omega$ which forms a countable subset $\cbr{(w_q,r_q)}_{q=1}^{\infty}$ and induces a countable subset of $\cG_{\QQ}=\cbr{g_{w_q,r_q}(x,a)}_{q=1}^{\infty}$. By Theorem 3.4.5 in \citet{gine2021mathematical}, we have,
for any $\delta>0$,
\[
\PP\rbr{\sup_{q}|g_{w_q,r_q}(x,a)-\int_{\cX,\cA}g_{w_q,r_q}(x,a)dQ|\ge 2\EE\sbr{\sup_{q}\frac{1}{n}\sum_{j=1}^{n}\epsilon_jg_{w_q,r_q}(x_j,a_j)}+\sqrt{\frac{2\log(1/\delta)}{n}}}\le \delta.
\]
By the density of rational numbers in real numbers and $g$ is $2L_0$-Lipschitz continuous in $(w,r)$, we know that 
\[
\PP\rbr{\sup_{g\in\cG}\Big|g_{w,r}(x,a)-\int_{\cX,\cA}g_{w,r}(x,a)dQ|\ge 2\EE\sbr{\sup_{g\in\cG}\frac{1}{n}\sum_{j=1}^{n}\epsilon_jg_{w,r}(x_j,a_j)}+\sqrt{\frac{2\log(1/\delta)}{n}}}\le \delta.
\]

Now we try to bound the Rademacher complexity term $\EE\sbr{\sup_{g\in\cG}\frac{1}{n}\sum_{j=1}^{n}\epsilon_jg_{w,r}(x_j,a_j)}$.
Recall that we have $\cN(t,\Omega^2,||\cdot||_{\infty})\le (\frac{A}{t})^{2d}$, $||g_{w_1,r_1}(x,a)-g_{w_2,r_2}(x,a)||_{\infty}\le 2L_0||(w_1,r_1)-(w_2,r_2)||_{\infty}$. Then,
\[
\cN(t,\cG,||\cdot||_{\infty})\le (\frac{2L_0A}{t})^{2d}.
\]
Then, by Dudley's integral entropy bound and using $||g||_{\infty}\le 1$, we have
\[
\EE_{\varepsilon}\sup_{g\in\cG}\frac{1}{n}\sum_{j=1}^{n}\epsilon_jg(x_j,a_j)\le \frac{12}{\sqrt{n}}\int_{0}^{1}\sqrt{\log\cN(y,||\cdot||_{\infty},\cG)}dy\le \frac{24\sqrt{d}}{\sqrt{n}}\sqrt{\log(2L_0A)}.
\]
For the simplicity of notation, we use $Q(\cK^{x,a}(w,r))$ to denote $\int_{\cX,\cA}\int_{S}\phi(x,a,w,s)\phi(x,a,r,s)dm(s)dQ$ and $Q_n(\cK^{x,a}(w,r))$ to denote its empirical counterpart.
Combining all the parts above together, we have that with probability at least $1-\delta$,
\[
\sup_{(w,r)}\abr{Q(\cK^{x,a}(w,r))-Q_n(\cK^{x,a}(w,r))}\le \frac{48\sqrt{d\log(2L_0A)}}{\sqrt{n}}+\frac{\sqrt{2\log(1/\delta)}}{\sqrt{n}}.
\]
By \Cref{ass:nonsingular}, we get with probability at least $1-\delta$,
\[
\sup_{(w,r)}\abr{\frac{Q(\cK^{x,a}(w,r))}{Q_n(\cK^{x,a}(w,r))}-1}\le \frac{48\sqrt{d\log(2L_0A)}}{\eta\sqrt{n}}+\frac{\sqrt{2\log(1/\delta)}}{\eta\sqrt{n}}.
\]


Recalling that $\cU_{\cD}=\sum_{j=1}^{n}\int_{\Omega}\cK^{x_j,a_j}(w,r)dm(s)d\nu(r)$, our bound in \Cref{cor:reg_fixed_design} yields with probability at least $1-\delta$, the following bound holds that
$$||\hat{\theta}_{\cD}-\theta^*||_{\cU_{\cD}}\le\cE_{\delta}(n).$$
We use $\cU^*$ to denote the integral operator induced by kernel $n\int_{\cX,\cA}\int_{S}\phi(x,a,w,t)\phi(x,a,r,t)dm(s)dQ_{x,a}$ and $\cU_{\cD}$ is its empirical counterpart. Therefore, for any $\theta$, we have that
\begin{align*}
    \inner{\theta}{\cU^*(\theta)}&=\int_{\Omega}\theta(w)\int_{\Omega}n\int_{\cX,\cA}\int_{S}\theta(r)\phi(x,a,w,t)\phi(x,a,r,t)dm(s)dQ_{x,a}d\nu(r)d\nu(w)\\
    &=n\int_{\Omega}\int_{\Omega}\theta(w)\theta(r)Q(\cK(w,r))d\nu(w)d\nu(r),
\end{align*}
\begin{align*}
    \inner{\theta}{\cU_{\cD}(\theta)}&=\int_{\Omega}\theta(w)\int_{\Omega}\sum_{j=1}^{n}\int_{S}\theta(r)\phi(x_j,a_j,w,s)\phi(x_j,a_j,r,s)dm(s)d\nu(w)d\nu(r)\\
    &=n\int_{\Omega}\int_{\Omega}\theta(w)\theta(r)Q_n(\cK(w,r))d\nu(w)d\nu(r).
\end{align*}
Combining these two equations and the concentration analysis above, we conclude that with probability at least $1-\delta$,
\begin{align*}
    \inner{\theta}{\cU^*(\theta)}&=n\int_{\Omega}\int_{\Omega}\theta(w)\theta(r)Q(\cK(w,r))d\nu(w)d\nu(r)\\
    &=n\int_{\Omega}\int_{\Omega}\theta(w)\theta(r)\cbr{\frac{Q(\cK(w,r))}{Q_n(\cK(w,r))}}Q_n(\cK(w,r))d\nu(w)d\nu(r)\\
    &\le \rbr{1+\frac{48\sqrt{d\log(2L_0A)}+2\sqrt{\log(1/\delta)}}{\eta\sqrt{n}}} n\int_{\Omega}\int_{\Omega}\theta(w)\theta(r)Q_n(\cK(w,r))d\nu(w)d\nu(r)\\
    &= \rbr{1+\frac{48\sqrt{d\log(2L_0A)}+2\sqrt{\log(1/\delta)}}{\eta\sqrt{n}}}\inner{\theta}{\cU_{\cD}(\theta)}.
\end{align*}
Replacing the general $\theta$ by $\hat{\theta}_{\cD}-\theta^*$ an using \Cref{cor:reg_fixed_design}, it holds that, with probability at least $1-2\delta$,
\begin{align*}
    ||\hat{\theta}_{\cD}-\theta^*||_{\cU^*}\le\rbr{1+\frac{48\sqrt{d\log(2L_0A)}+2\sqrt{\log(1/\delta)}}{\eta\sqrt{n}}}^{1/2} \cE_{\delta}(n).
\end{align*}
On the other hand, by directly calculating $||\hat{\theta}_{\cD}-\theta^*||^2_{\cU^*}$, we find that
\begin{align*}
    ||\hat{\theta}_{\cD}-\theta^*||^2_{\cU^*}=&\inner{\hat{\theta}_{\cD,\varepsilon}-\theta^*}{\cU^*(\hat{\theta}_{\cD}-\theta^*)}\\
    =&n\int_{\cX,\cA}\int_{\Omega}\rbr{\hat{\theta}_{\cD}-\theta^*}(w)\int_{\Omega}\rbr{\hat{\theta}_{\cD}-\theta^*}(r)\phi(x,a,w,t)\phi(x,a,r,t)dm(s)d\nu(w)d\nu(r)dQ_{x,a}\\
    =&n\int_{\cX,\cA}\int_{S}\rbr{\int_{\Omega}\rbr{\hat{\theta}_{\cD}-\theta^*}(w)\phi(x,a,w,t)d\nu(w)}\rbr{\int_{\Omega}\rbr{\hat{\theta}_{\cD}-\theta^*}(r)\phi(x,a,r,t)d\nu(r)}dm(s)dQ_{x,a}\\
    =&n\EE_{x,a}\sbr{||\hat{F}_{\cD}(x,a,s)-F^*(x,a,s)||_{L^2(S,m)}^2}.
\end{align*}

Assembling all these parts, we finally get the following inequality,
$$\EE_{x,a}\sbr{||\hat{F}_{\cD}(x,a,s)-F^*(x,a,s)||_{L^2(S,m)}^2}\le \frac{\rbr{1+(48\sqrt{d\log(2L_0A)}+2\sqrt{\log(1/\delta)})/\eta}\cE_{\delta}(n)^2}{n}.$$
By \Cref{lemma:bound_cE_D}, we know that 
$$\cE_{\delta}(n):=2\log(1/\delta)^{1/2}+\rbr{2\sqrt{s_0\log(1+n)}+M}n^{\frac{\gamma}{\gamma+2}}.$$
By denoting $C(d,L_0,\delta,A,\eta)$ as $\rbr{1+(48\sqrt{d\log(2L_0A)}+2\sqrt{\log(1/\delta)})/\eta}$, we finish our proof.
\end{proof}

\section{Proofs in \Cref{sec:alg}}\label{sec:proofs_alg}
\begin{lemma}[\citep{simchi2020bypassing}]\label{lemma:general_guarantte_IGW}
Assume that we are given an offline regression oracle $\mathsf{RegOff}$ and i.i.d. data $\cD=\cbr{(x_i,a_i,r_i)}_{i=1}^{n}$ where $\EE[r_i|x_i,a_i]=f^*(x_i,a_i)$. With probability at least $1-\delta$, it returns
    $\hat{f}:\cX\times\cA\rightarrow \RR$ such that
    $$\EE_{x,a}\sbr{\rbr{\hat{f}(x,a)-f^*(x,a)}^2}\le\frac{\mathsf{Est}_{\delta}(n)}{n}$$
for some number $\mathsf{Est}_{\delta}(n)$. Then, define epoch schedule $\xi_m=2^m$ and exploration parameter $\varsigma_m=\frac{1}{2}\sqrt{K/\mathsf{Est}_{\delta/m^2}(\xi_{m-1})}$. For any $t$, let $m(t)$ be the number of epochs that round $t$ lies in. For any $T$ large enough, with probability at least $1-\delta$, the regret of \Cref{alg:IGW_pseudo_code} after $T$ rounds is at most 
$$\cO\rbr{\sqrt{K}\sum_{m=2}^{m(T)}\sqrt{\frac{\mathsf{Est}_{\delta/(2m^2)}(\xi_{m-1}-\xi_{m-2})}{\xi_{m-1}-\xi_{m-2}}}(\xi_m-\xi_{m-1})}.$$
\end{lemma}

\begin{proof}[Proof of \Cref{thm:theoretical_guarantee}]
   
   From our regression oracle, we have that
   \begin{align*}
       \frac{\mathsf{Est}_{\delta}(n)}{n}&=\frac{L^2C(d,L_0,\delta/2,A,\eta)\cE_{\delta/2}(n)^2}{n}
   \end{align*}
   By applying \Cref{lemma:general_guarantte_IGW} and plugging in the choices that $\xi_m=2^m$, we get
   \begin{align}\label{equa:bound_regret_1}
       &\text{Reg}(T)\nonn\\
       \le& \cO\rbr{\sqrt{K}\sum_{m=2}^{m(T)}\sqrt{\frac{\mathsf{Est}_{\delta/(2m^2)}(\xi_{m-1}-\xi_{m-2})}{\xi_{m-1}-\xi_{m-2}}}(\xi_m-\xi_{m-1})}+\cO(1)\nonn\\
       \le&\cO\rbr{\sqrt{K}L\sum_{m=2}^{m(T)}\sqrt{C(d,L_0,\delta/4m^2,A,\eta)}\frac{\cE_{\delta/(4m^2)}(\xi_{m-1}-\xi_{m-2})}{\sqrt{\xi_{m-1}-\xi_{m-2}}}(\xi_{m}-\xi_{m-1})+1}\nonn\\
       \le&\cO\rbr{\sqrt{K}L\sum_{m=2}^{m(T)}\sqrt{C(d,L_0,\delta/4m^2,A,\eta)}\frac{2\log(4m^2/\delta)^{1/2}+\rbr{2\sqrt{s_0\log(1+2^{m-2})}+M}(2^{m-2})^{\frac{\gamma}{\gamma+2}}}{2^{m/2-1}}\cdot2^{m-1}+1}\nonn\\
   \end{align}
   By our epoch schedule, it holds that $m(T)\le \ceil{\log_2(T)}$. By \Cref{thm:reg_random_design}, we know 
   \[
   C(d,L_0,\delta/4m^2,A,\eta)=\rbr{1+(48\sqrt{d\log(2L_0A)}+2\sqrt{\log(4m^2/\delta)})/\eta}\le \frac{192\sqrt{d\log(2L_0A)}}{\eta}\sqrt{2\log(m/\delta)}.
   \]
   Thus,
   \[
   C(d,L_0,\delta/4m^2,A,\eta)\le\cO\rbr{\frac{\sqrt{d\log(2L_0A)}}{\eta}\sqrt{\log(m/\delta)}}.
   \]
Therefore, for \Cref{equa:bound_regret_1}, we further have
   \begin{align*}
       \text{Reg}(T)
       \le& \cO\rbr{\frac{L\sqrt{dK\log(2L_0A)}}{\eta}\sqrt{\log(m/\delta)}\sum_{m=2}^{m(T)}\rbr{\rbr{\log(\log T/\delta)}^{1/2}+\rbr{\sqrt{s_0\log(1+\log T)}+M}2^{\frac{\gamma(m-2)}{(\gamma+2)}}}2^{m/2}}.
   \end{align*}
   Denote the number $\frac{L\sqrt{dK\log(2L_0A)}}{\eta}$ as $\Tilde{C}(K,L,L_0,A,d,\eta)$. We have,
   \begin{align*}
       &\cO\rbr{\frac{L\sqrt{dK\log(2L_0A)}}{\eta}\sqrt{\log(m/\delta)}\sum_{m=2}^{m(T)}\rbr{\rbr{\log(\log T/\delta)}^{1/2}+\rbr{\sqrt{s_0\log(1+T)}+M}2^{\frac{\gamma(m-2)}{(\gamma+2)}}}2^{m/2}}\\
       \le&\Tilde{C}(K,L,L_0,A,d,\eta)\cO\rbr{\log(\log T/\delta)\sum_{m=2}^{m(T)}2^{m/2}+(\log\log T/\delta)^{1/2}\sum_{m=2}^{m(T)}\rbr{\sqrt{s_0\log(1+\log T)}+M}2^{\frac{(3\gamma+2)m}{2(\gamma+2)}}}\\
       \le&\Tilde{C}(K,L,L_0,A,d,\eta)\cO\rbr{\log(\log T/\delta)\sqrt{T}+\rbr{s_0^{1/2}+M}(\log\log T/\delta)T^{\frac{3\gamma+2}{2(\gamma+2)}}}\\
       \le& \Tilde{C}(K,L,L_0,A,d,\eta)\rbr{(s_0^{1/2}+M)\vee 1}\cO\rbr{\log(\log T/\delta)\rbr{\sqrt{T}+T^{\frac{3\gamma+2}{2(\gamma+2)}}}}.
   \end{align*}
   \end{proof}
\begin{proof}[Proof of \Cref{thm:utility_bound}]
       By the $L$-Lipschitz continuity of $\cT$, we have
\begin{align*}
\EE_{x,a}\sbr{\rbr{\cT(\hat{F}_{\cD}(x,a,s))-\cT(F^*(x,a,s))}^2}&\le\EE_{x,a}\sbr{L^2||\hat{F}_{\cD}(x,a,s)-F^*(x,a,s)||_{L^2(S,m)}^2}\\
\end{align*}
Then we apply \Cref{thm:reg_random_design} and replace $2\delta$ by $\delta
$ in the claim of \Cref{thm:reg_random_design}. Thus, with probability $1-\delta$, 
\[
\EE_{(x,a)\sim Q}\sbr{||\hat{F}_{\cD}(x,a,s)-F^*(x,a,s)||_{L^2(S,m)}^2}\le \frac{C(d,L_0,\delta/2,A,\eta)\cE_{\delta/2}(n)^2}{n}.
\]
Combining these two inequalities together, we shall get our result.
\end{proof}

\section{Computation}\label{sec:computation}
Given the theorems and properties we have established, a natural and fundamental question arises: how can we numerically compute the eigenvalues and eigenfunctions of an infinite-dimensional operator?  In fact, within the applied functional analysis community, there is a wide range of methods for numerically solving the eigenvalues and eigenfunctions of integral operators in infinite-dimensional spaces (for example, see \citet{chatelin2011spectral,ray2013numerical, chatelin1981spectral, kohn1972improvement, ahues2001spectral,panigrahi2017galerkin} for reference), including the Galerkin method, the Rayleigh-Ritz method, quadrature approximation methods and so on. In this section, we will briefly introduce one method based on degenerate kernels for calculating the eigenvalues and eigenfunctions as an illustration. Please see \citet{gnaneshwar2007degenerate} for full details.

We assume that $\Omega=[0,1]$ and an integral operator $C$ is defined as $C[x](s)=\int_{0}^{1}k(s,t)x(t)dt$. 
The basic idea of \Cref{alg:degenerate_kernel_pseudo_code} is to construct a finite-dimensional integral kernel $k_{N_h}$ to approximate $k$ holding the property that the differences between their eigenvalues and eigenfunctions are small enough to be ignored. Therefore, to solve the eigenvalues and eigenfunctions for the kernel 
$k$, we can instead solve them for 
$k_{N_h}$, which is finite-dimensional and can be readily transformed into a matrix eigenvalue problem.

We present the pseudo-code in \Cref{alg:degenerate_kernel_pseudo_code} and omit the theoretical approximation analysis here. Please refer to \citet{gnaneshwar2007degenerate} for concrete theoretical guarantees.

\begin{algorithm}[ht]
\caption{Degenerate Kernel Method}\label{alg:degenerate_kernel_pseudo_code}
    \begin{algorithmic}
       \Require partition number $n$, kernel $k(s,t)$, degree number $r$ and corresponding Legendre polynomials $L_r(t)=\frac{d^r}{dt^r}(t^2-1)^r$.
       \State Partition $[0,1]$ into $n$ intervals, $0=x_0<x_1<x_2<\cdots<x_{n-1}<x_n=1$, $I_k=[x_{k-1},x_k]$, $h=\max|I_k|$. 
       Set $f_k(t)=\frac{1-t}{2}x_{k-1}+\frac{1+k}{2}x_k,-1\le t\le 1$.
       \State Define $\cP_r$ as the space of polynomials of degree $\le r-1$. Denote $N_h=nr$.
       \State Define piecewise polynomial space associated with $\cP_r$ $$\cS_h^r=\cbr{u:[0,1]\rightarrow \RR:u\big|_{I_k}\in\cP_r,1\le k\le n}.$$
       \State Find the Gauss point set (zeros set) of $L_r(t)$ in $[0,1]$, $B_r=\cbr{y_1\cdots,y_r}$.
       \State Let $A=\cup_{k=1}^{n}f_k(B_r)=\cbr{\omega_{i,k}=f_k(y_i):1\in[r],k\in[n]}$ and $l_i(x)$ be the Lagrange polynomials of degree $r-1$ with respect to $y_1\cdots,y_r$ such that $l_i(y_j)=\delta_{ij}$.
       \State Define $\rho_{jp}(x)=\begin{cases}
      l_j(f_p^{-1}(x))& x\in[x_{p-1},x_p]\\
      0 & \text{otherwise}.
    \end{cases}$
    \State Notice $\rho_{jp}\in\cS_h^r$ and $\rho_{jp}(\omega_{ik})=\delta_{ji}\delta_{kp},\ \text{for }i,j\in[r],k,p\in[n]$.
    Set $t_{(k-1)r+j}=\omega_{jk}$, $z_{(k-1)r+j}=\rho_{jk},\ \text{for }k,p\in[n]$.
    \State Let $\cS_h^r\otimes\cS_h^r$ be the tensor product space of $\cS_h^r$ with dimension $N_h=nr$.
    \State Define degenerate kernel
    $$k_{N_h}(s,t)=\sum_{i=1}^{N_h}\sum_{j=1}^{N_h}k(\omega_i,\omega_j)z_i(s)z_j(t),$$ which induces a degenerate kernel operator 
    $$C_{N_h}(x)(s)=\int_{0}^{1}\sum_{i=1}^{N_h}\sum_{j=1}^{N_h}k(\omega_i,\omega_j)z_i(s)z_j(t)x(t)dt.$$
    \State Solve the eigenvalue problem for the $N_h$-dimensional integral kernel $k_{N_h}$.
    \end{algorithmic}
\end{algorithm}

We now provide a method to solve the eigenvalue problem for the finite-dimensional kernel $k_{N_h}$ for completeness. 

\paragraph{Solve eigenvalue problem for finite-dimensional $k_{N_h}$:}Consider the following formula that
\begin{equation}\label{equa:computation_1}
    \lambda\cdot g(s)=\int_{0}^{1}\sum_{i=1}^{N_h}\sum_{j=1}^{N_h}k(\omega_i,\omega_j)z_i(s)z_j(t)g(t)dt.
\end{equation}
Here, $\lambda$ stands for the unknown eigenvalue and $g$ is the corresponding eigenfunction.
Define $$c_i:=\frac{1}{\lambda}\int_{0}^{1}\sum_{j=1}^{N_h}k(\omega_i,\omega_j)z_j(t)g(t)dt,\ i\in[N_h],$$ and $\bc=\rbr{c_1,c_2,\cdots,c_{N_h}}$ is an unknown vector in $\RR^{N_h}$ which we will solve for. Then, we can rewrite \Cref{equa:computation_1} as
\begin{equation}\label{equa:computation_2}
    g(s)=\sum_{i=1}^{N_h}c_iz_j(s).
\end{equation}
Plugging \Cref{equa:computation_2} back to the definition of $c_i$, it holds that
\begin{equation}\label{equa:computation_3}
    c_i=\frac{1}{\lambda}\int_{0}^{1}\sum_{j=1}^{N_h}k(\omega_i,\omega_j)z_j(t)\sum_{k=1}^{N_h}c_kz_k(t)dt=\frac{1}{\lambda}\sum_{j=1}^{N_h}\sum_{k=1}^{N_h}c_kk(\omega_i,\omega_j)\int_{0}^{1}z_j(t)z_k(t)dt,\ i\in[N_h],
\end{equation}
which is equivalent to
\begin{equation}\label{equa:computation_4}
    \lambda \cdot c_i=\int_{0}^{1}\sum_{j=1}^{N_h}k(\omega_i,\omega_j)z_j(t)\sum_{k=1}^{N_h}c_kz_k(t)dt=\sum_{j=1}^{N_h}\sum_{k=1}^{N_h}c_kk(\omega_i,\omega_j)\int_{0}^{1}z_j(t)z_k(t)dt,\ i\in[N_h].
\end{equation}
Note that within the implementation of \Cref{alg:degenerate_kernel_pseudo_code}, both $k(\omega_i,\omega_j)$ and $z_j(t),z_k(t)$ are known to us. Besides, the right hand side of \Cref{equa:computation_4} is linear with respect to $\bc$.

Define $b_{ik}=\sum_{j=1}^{N_h}k(\omega_i,\omega_j)\int_{0}^{1}z_{j}(t)z_k(t)dt$. Then, \Cref{equa:computation_4} is equivalent to 
\[
\lambda\cdot\begin{pmatrix}c_1\\
c_2\\
\vdots\\
c_{N_h}    
\end{pmatrix}=\begin{pmatrix}
    b_{11}& b_{12}&\cdots&b_{1N_{h}}\\
    b_{21}& b_{22}&\cdots&b_{2N_{h}}\\
    \vdots&\vdots&\vdots&\vdots\\
    b_{N_{h}1}& b_{N_{h}2}&\cdots&b_{N_{h}N_{h}}\\
\end{pmatrix}\cdot\begin{pmatrix}c_1\\
c_2\\
\vdots\\
c_{N_h}    
\end{pmatrix}.
\]

\Cref{equa:computation_4} finally turns out to be a matrix eigenvalue problem with unknown eigenvector $\bc=(c_1,c_2\cdots,c_{N_h})^{T}$ and eigenvalue $\lambda$. We could solve this matrix eigenvalue problem to acquire $(c_1,c_2\cdots,c_{N_h})$ and eigenvalue $\lambda$. 
Eventually, we derive the eigenfunction associated with $\lambda$ for kernel $k_{N_h}$ using \Cref{equa:computation_2}, i.e., 
$g(s)=\sum_{i=1}^{N_h}c_iz_i(s).$

\section{Auxiliary Results}\label{sec:auxiliary}
\begin{definition}[\citet{gohberg2012traces}]
    Let $A$ be a compact operator in Hilbert space $H$ and let $\lambda_1(A^*A)\ge \lambda_2(A^*A)\ge\cdots$
    be the sequence of non-zero eigenvalues of $A^*A$. $A^*$ is the adjoint operator of $A$. Then the $j$ th singular value of $A$ $s_j(A)$ is defined as
    $$s_j(A):=(\lambda_j(A^*A))^{1/2}.$$
\end{definition}
\begin{property}
    For any self-adjoint positive Hilbert-Schmidt integral operator $\cU$, it holds that
    $$\lambda_j(\cU)=s_j(\cU)\ \text{for }\forall\ j.$$
\end{property}
\begin{proof}
    We denote the integral kernel of $\cU$ as $\cK$. Because $\cU$ is self-adjoint positive Hilbert-Schmidt integral operator, then for $\forall\ \theta$,
    $$\cU(\theta)=\sum_{i=1}^{\infty}\lambda_i(\cU)\inner{\theta}{e_i}e_i.$$
    Therefore, 
    $$\cU(\cU(\theta))=\sum_{i=1}^{\infty}\lambda_i^2\inner{\theta}{e_i}e_i.$$
    Because $\cU^*=\cU$, we finish the proof.
\end{proof}
\begin{theorem}[Corollary 3.6 in~\citet{gohberg2012traces}]
    If $A$ and $B$ are compact operators on the Hilbert space $H$, then for any $n$,
    $$\sum_{j=1}^{n}s_j(A+B)\le\sum_{j=1}^{n}s_j(A)+\sum_{j=1}^{n}s_j(B).$$
\end{theorem}
\begin{theorem}[Corollary 4.2 in~\citet{gohberg2012traces}]
    Let $A$ and $B$ be compact operators on a Hilbert space $H$, then for any $p>0$ and any $k$,
    $$\sum_{j=1}^{k}s_j^p(AB)\le \sum_{j=1}^{k}s_j^p(A)s_j^p(B),$$
    and
    $$\sum_{j=1}^{k}s_j(AB)\le \sum_{j=1}^{k}s_j(A)s_j(B).$$
\end{theorem}
Now we introduce the functional determinant of a trace class operator. We notice here that the formal definition of functional determinant requires exterior product and tensorization of the Hilbert space $H$. Nonetheless, by Lidskii's Theorem~\citep{simon2005trace}, we are able to achieve an equivalent characterization of it, so we directly use this characterization as a definition for simplicity.
\begin{definition}[Theorem 3.4.7 in~\citet{kostenkotrace}]
For any trace class operator $A$, the functional determinant $\det(I+zA)$ is defined as
$$\det(I+zA):=\prod_{i}(1+z\lambda_i(A)),\ \forall z\in\CC.$$
\end{definition}
\begin{theorem}[Theorem 3.10 in ~\citet{simon2005trace}]\label{det_formula}
    Suppose the integral operator $A$ is defined as $Af(x)=\int_{\Omega}K(x,y)f(y)dy$, where $\Omega$ is compact and $K(x,y)$ is continuous on $\Omega\times\Omega$. Then, it holds that
    $$tr(A)=\int_{\Omega}K(x,x)dx.$$
    In the meanwhile,
    $$\det(I+A)=\sum_{i=0}^{\infty}\frac{\alpha_m}{m!},$$
    where $$\alpha_m=\int_{\Omega}\int_{\Omega}\cdots\int_{\Omega} \det[(K(x_i,x_j))_{1\le i,j\le m}]dx_1dx_2\cdots dx_m,\ \forall\ m\ge 1.$$
    We conventionally set $\alpha_0=1$.
\end{theorem}

Now we are ready to analyze the integral operator $\cL^{x,a}$ and $\cU_{\cD}$ in our paper.
\begin{lemma}\label{bound_single_L}
    For any $x,a$, $\det(1+\cL^{x,a})\le e$.
\end{lemma}
\begin{proof}[Proof of \Cref{bound_single_L}]

    We apply \Cref{det_formula} to the operator $\cL^{x,a}$. The associated integral kernel is $\cK^{x,a}(w,r)=\int_{S}\phi(x,a,w,t)\phi(x,a,r,t)dm(t)$.
    Therefore, by \Cref{det_formula}, we have
    \begin{align*}
        &(K(w_i,w_j))_{i,j=1}^{m}
        =\begin{pmatrix}
        \int\phi(x,a,w_1,t)^2dt&\int\phi(x,a,w_1,t)\phi(x,a,w_2,t)dt&\cdots&\int\phi(x,a,w_1,t)\phi(x,a,w_m,t)dt\\
        \int\phi(x,a,w_2,t)\phi(x,a,w_1,t)dt&\int\phi(x,a,w_2,t)^2dt&\cdots&\int\phi(x,a,w_2,t)\phi(x,a,w_m,t)dt\\
        \vdots&\vdots&\vdots&\vdots\\
        \int\phi(x,a,w_m,t)\phi(x,a,w_1,t)dt&\int\phi(x,a,w_m,t)\phi(x,a,w_2,t)dt&\cdots&\int\phi(x,a,w_m,t)^2dt\\
    \end{pmatrix}
    \end{align*}
    This is a symmetric matrix and 
    $$\sum_{i=1}^{m}\lambda_i((K(w_i,w_j))_{i,j=1}^{m})=\sum_{i=1}^{m}\int_{S}\phi(x,a,w_i,t)^2dm(t)\le m,$$
    due to $0\le\phi\le1$ and $m(S)=1$.
    Therefore, by the AM-GM inequality,
    $$\rbr{\prod_{i=1}^{m}\lambda_i((K(w_i,w_j))_{i,j=1}^{m})}^{1/m}\le \frac{\sum_{i=1}^{m}\lambda_i((K(w_i,w_j))_{i,j=1}^{m})}{m}\le 1.$$
    We find that
    $$\det[(K(w_i,w_j))_{i,j=1}^{m}]=\prod_{i=1}^{m}\lambda_i((K(w_i,w_j))_{i,j=1}^{m})\le1.$$
    Recalling $\nu(\Omega)=1$, we have that for 
    $\forall\ m\ge1,\ \alpha_m\le \int_{\Omega}\int_{\Omega}\cdots\int_{\Omega}1d\nu(w_1)\cdots d\nu(m)\le 1.$ Thus,
    $$\det(1+\cL^{x,a})=\sum_{i=0}^{\infty}\frac{\alpha_m}{m!}\le \sum_{i=0}^{\infty}\frac{1}{m!}\le e.$$
\end{proof}

\end{document}